\documentclass[11pt,twoside]{article}

\usepackage{fullpage}
\usepackage{epsf}
\usepackage{fancyhdr}
\usepackage{graphics}
\usepackage{graphicx}
\usepackage{psfrag}
\usepackage[T1]{fontenc}
\usepackage{color}
\usepackage{amsthm}
\usepackage{times}
\usepackage{amsfonts}
\usepackage{amsmath}
\usepackage{amssymb}
\usepackage{mathrsfs}
\usepackage{bm}
\usepackage{accents}
\usepackage{listings}
\usepackage{caption}
\usepackage{mleftright}
\usepackage{subcaption}
\usepackage[linesnumbered, ruled, vlined]{algorithm2e}
\usepackage[titletoc,toc,title]{appendix}
\usepackage{enumerate}
\usepackage{verbatim} % verbatim input
\usepackage{csquotes} % context sensitive quotes
\usepackage{upquote} % allow correct quotes in verbatim
\usepackage[bottom]{footmisc} % put footnotes down at the bottom margin
\usepackage{thmtools}
\usepackage[dvipsnames]{xcolor}
\usepackage[english]{babel} % multilinguality
\usepackage[shortlabels]{enumitem}
\usepackage[colorlinks,linkcolor = RawSienna, urlcolor  = RedViolet, citecolor = RoyalBlue, anchorcolor = ForestGreen,bookmarks=False]{hyperref}
\newlength{\widebarargwidth}
\newlength{\widebarargheight}
\newlength{\widebarargdepth}

\makeatletter
\long\def\@makecaption#1#2{
        \vskip 0.8ex
        \setbox\@tempboxa\hbox{\small {\bf #1:} #2}
        \parindent 1.5em  %% How can we use the global value of this???
        \dimen0=\hsize
        \advance\dimen0 by -3em
        \ifdim \wd\@tempboxa >\dimen0
                \hbox to \hsize{
                        \parindent 0em
                        \hfil 
                        \parbox{\dimen0}{\def\baselinestretch{0.96}\small
                                {\bf #1.} #2
                                } 
                        \hfil}
        \else \hbox to \hsize{\hfil \box\@tempboxa \hfil}
        \fi
        }
\makeatother

\makeatletter
\newcounter{manualsubequation}
\renewcommand{\themanualsubequation}{\alph{manualsubequation}}
\newcommand{\startsubequation}{%
  \setcounter{manualsubequation}{0}%
  \refstepcounter{equation}\ltx@label{manualsubeq\theequation}%
  \xdef\labelfor@subeq{manualsubeq\theequation}%
}
\newcommand{\tagsubequation}{%
  \stepcounter{manualsubequation}%
  \tag{\ref{\labelfor@subeq}\themanualsubequation}%
}
\let\subequationlabel\ltx@label
\makeatother

\renewcommand{\baselinestretch}{1.04} % stretch distance between baselines
\date{}
\usepackage[style = alphabetic,citestyle=alphabetic,maxbibnames=99,backend=biber,sorting=nyt,natbib=true,backref=true]{biblatex}
\DefineBibliographyStrings{english}{%
  backrefpage = {Cited on page},% originally "cited on page"
  backrefpages = {Cited on pages},% originally "cited on pages"
}
\newcommand{\BlackBox}{\rule{1.5ex}{1.5ex}}  % end of proof

\makeatletter
\renewenvironment{proof}[1][\proofname]{%
  \par\pushQED{\qed}\normalfont%
  \topsep6\p@\@plus6\p@\relax
  \trivlist\item[\hskip\labelsep\bfseries#1\@addpunct{.}]%
  \ignorespaces
}{%
  \popQED\endtrivlist\@endpefalse
}
\makeatother

\newtheorem{theorem}{Theorem}[section]
\newtheorem{lemma}[theorem]{Lemma}

\newtheorem{remark}[theorem]{Remark}

\newtheorem{corollary}[theorem]{Corollary}
\newtheorem{definition}[theorem]{Definition}

\newcommand\numberthis{\addtocounter{equation}{1}\tag{\theequation}}

\newcommand{\Tr}{\mathrm{Tr}}

\newcommand{\by}{{\bf y}}

\newcommand{\cL}{{\cal L}}

\newcommand{\cN}{{\cal N}}

\newcommand{\R}{\mathbb{R}}
\renewcommand{\S}{\mathbb{S}}
\newcommand{\N}{\mathbb{N}}

\newcommand{\E}{\mathbb{E}}

\renewcommand{\Pr}{\mathbb{P}}
\newcommand{\lv}{\lVert}
\newcommand{\rv}{\rVert}

\newcommand{\tw}{\widetilde{w}}

\renewcommand{\epsilon}{\varepsilon}

\DeclareSymbolFont{extraup}{U}{zavm}{m}{n}
\DeclareMathSymbol{\varheart}{\mathalpha}{extraup}{86}
\DeclareMathSymbol{\vardiamond}{\mathalpha}{extraup}{87}

\DeclareMathOperator*{\argmin}{arg\,min}

\newcommand{\htheta}{\hat{\theta}}
\newcommand{\ttheta}{\widetilde{\theta}}
\newcommand{\tby}{\widetilde{\by}}

\renewcommand{\epsilon}{\varepsilon}
\newcommand{\beps}{\bm{\epsilon}}

\DeclareMathSizes{9}{8}{7}{5}

\title{\textbf{The Interplay Between \\
Implicit Bias and Benign Overfitting \\ 
in Two-Layer Linear Networks}}

\author{
Niladri S. Chatterji \\ 
Computer Science Department \\
Stanford University \\
niladri@cs.stanford.edu \\
      \and
Philip M. Long  \\
Google \\
plong@google.com \\
      \and
Peter L. Bartlett  \\
University of California, Berkeley \& Google \\
peter@berkeley.edu \\
}

\date{\today}

\begin{document}
\maketitle
\begin{abstract}
The recent success of neural network models has shone light on a rather surprising statistical phenomenon: statistical models that perfectly fit noisy data can generalize well to unseen test data. Understanding this phenomenon of \emph{benign overfitting} has attracted intense theoretical and empirical study. In this paper, we consider interpolating two-layer linear neural networks trained with gradient flow on the squared loss and derive bounds on the excess risk when the covariates satisfy sub-Gaussianity and anti-concentration properties, and the noise is independent and sub-Gaussian. By leveraging recent results that characterize the implicit bias of this estimator, 
our bounds emphasize the role of both the quality of the initialization as well as the properties of the data covariance matrix in achieving low excess risk.
\end{abstract}

%JMLR 
% \begin{keywords}
%  implicit bias, generalization, benign overfitting, interpolation,
%  neural networks, regression
%  \end{keywords}

%JMLR
\tableofcontents
\section{Introduction} 
Understanding benign overfitting---the phenomenon where statistical models predict well on test data despite perfectly fitting noisy training data~\citep[see, e.g.,][]{zhang2016understanding,belkin2019reconciling,bartlett2021deep,belkin2021fit}---has recently attracted intense attention. One line of work has focused on understanding this phenomenon in relatively simple models such as linear regression~\citep{kobak2020optimal,hastie2019surprises,bartlett2020benign,muthukumar2020harmless,negrea2020defense,chinot2020robustness,NEURIPS2020_72e6d323,tsigler2020benign,bunea2020interpolation,chinot2020robustness_min,koehler2021uniform}
including with random features
\citep{hastie2019surprises,DBLP:conf/icml/YangYYSM20,li2021towards}, linear classification~\citep{montanari2019generalization,chatterji2020finite,liang2020precise,muthukumar2020classification,hsu2020proliferation,deng2019model,wang2021benign}, kernel regression~\citep{liang2020just,mei2019generalization,liang2020MultipleDescent} and simplicial nearest neighbor methods~\citep{belkin2018overfitting}.

A complementary
line of work~\citep{soudry2018implicit,ji2018risk,gunasekar2017implicit,nacson2018stochastic,gunasekar2018implicit,gunasekar2018characterizing,DBLP:conf/iclr/YunKM21,azulay2021implicit} has formalized the argument~\citep{neyshabur2017implicit}
that, even when no explicit regularization is used in training these models, 
there is nevertheless implicit regularization encoded in the choice of the optimization method, loss function and initialization. They argue that this implicit bias is critical in determining the generalization properties of the learnt 
model. 

Recently, \citet{azulay2021implicit}
characterized the implicit bias of
gradient flow applied to two-layer
linear neural networks with the squared loss.
More concretely, the setting is as follows. Given $n$ data points $(x_1,y_1),\ldots,(x_n,y_n) \in \R^{p}\times \R$, let $\by := (y_1,\ldots,y_n)^{\top} \in \R^n$ and $X := (x_1,\ldots,x_n)^{\top} \in \R^{n\times p}$. They studied two-layer linear networks, with $m$ hidden units, and weights $a \in \R^{m}$ and $W \in \R^{m\times p}$, that map an input $x\in \R^{p}$ to the scalar
\begin{align*}
   a^{\top}Wx.
\end{align*}
Let $\theta = a^{\top}W \in \R^{p}$ denote the 
standard parameterization of the
resulting linear map. A two-layer linear network with parameters $\{a,W\}$ is said to be \emph{balanced} if 
\begin{align*}
    aa^{\top} - WW^{\top} = 0.
\end{align*}
\citet{azulay2021implicit} showed in Proposition~1 that, starting from a balanced initial point ($a(0),W(0)$), if the gradient flow converges to a solution that perfectly fits the data, then the solution can be characterized as follows:
\begin{align} \label{def:two_layer_implicit_bias}
    \htheta &\in \argmin_{\theta \in \R^{p}}\lv \theta \rv^{3/2} - \frac{\theta(0)^{\top}\theta}{\sqrt{\lv \theta(0)\rv}}, \qquad \text{s.t., } \; \by=X\theta.
\end{align}

In this paper, we study the generalization properties of this solution in the overparameterized regime, where such
interpolation is possible. We prove upper bounds on the excess
risk and show that it depends both on the properties of the eigenstructure of population covariance matrix---as in the case of the minimum $\ell_2$-norm interpolant (ordinary least squares)~\citep{bartlett2020benign,tsigler2020benign}---and also on the quality of the initialization $\theta(0)$. In particular, we show that to drive the excess risk to zero, it suffices if the number of samples is large relative to the trace of the population covariance matrix and also that the number of ``small'' eigenvalues is large relative to $n$. Our bounds also show that the excess risk can be smaller as a rescaling of $\theta(0)$ gets closer to the optimal linear predictor.

An overview of the techniques that drive our analysis is as follows. We begin by showing that the predictor $\htheta$ can be viewed as a perturbation of the ordinary least squares solution in the subspace 
orthogonal to the row span of $X$.
To characterize this perturbation we find that it is important to derive upper and lower bounds on $\Tr((X X^\top)^{-1})$. To do this, as done in past work, we instead bound the trace of the ``tail'' of the matrix---the submatrix formed by the many low variance directions---and show that it not only concentrates but also provides a good approximation for the trace of the inverse of the entire matrix $\Tr((X X^\top)^{-1})$. 

Along the way we derive a new multiplicative high-probability lower bound on the least singular value of a non-isotropic rectangular random matrix (Lemma~\ref{l:smallest_singular_value}). We could not find such a result in the literature.   The most closely related work that we know of~\citep[see][and references therein]{rudelson2010non}, characterizing the ``hard edge'' of a random matrix, has focused on the most difficult case of isotropic square matrices.

The remainder of the paper is organized as follows. In Section~\ref{s:prelim} we introduce notation and 
definitions. In Section~\ref{s:main_results} we present our results. We provide a proof of our main result, Theorem~\ref{t:main}, in Section~\ref{s:proof_details} and prove our lower bound in Section~\ref{s:lower}. We conclude with a discussion in Section~\ref{s:discussion}.  

\section{Preliminaries}\label{s:prelim}
This section includes notational conventions and a description of the setting.
\subsection{Notation}
Given a vector $v$, let $\lv v \rv$ denote its Euclidean norm. Given a matrix $M$, let $\lv M \rv$ denote its Frobenius norm and $\lv M \rv_{op}$ denote its operator norm. 
  For any $j \in \N$, we denote the set $\{ 1,\ldots,j \}$ by $[j]$. Given a 
symmetric
matrix $M \in \R^{p\times p}$ we let $\mu_1(M)\ge \ldots \ge \mu_p(M)$ denote its eigenvalues. We let $I_p$ denote the identity matrix in $p$ dimensions. Given any vector $v \in \R^{p}$, we let $v_{1:j}\in \R^p$ denote the vector obtained by zeroing out the last $p-j$ coordinates of $v$ and let $v_{j+1:p}\in \R^{p}$ denote the vector obtained by zeroing out the first $j$ coordinates. Given a symmetric positive semidefinite matrix $M \in \R^{p\times p}$, let $M_{1:j} \in \R^{p\times p}$ be the matrix formed by zeroing out the last $p-j$ rows and columns of $M$, and let $M_{j+1:p} \in \R^{p\times p}$ be the matrix formed by zeroing out the first $j$ rows and columns. We let $ \lv v \rv_{M} := \sqrt{v^{\top}M v}$ denote the matrix norm of $v$ with respect to the matrix $M$.  We use the standard ``big Oh notation'' \citep[see, e.g.,][]{cormen2009introduction}. We will use $c,  c', c_1, \ldots$ to denote positive absolute constants, which may take different values in
different contexts.

\subsection{The setting}
Throughout the paper we assume that $p > n$. Although we assume throughout that the input dimension $p$ is finite, it is straightforward to extend our results to infinite $p$.

For random $(x,y)\in \mathbb{R}^p\times\mathbb{R}$, let
\[\theta^\star \in \argmin_{\theta \in\R^{p}} \E\left[(y-x^{\top}\theta)^2\right]\] 
be an arbitrary optimal linear regressor. We assume that $x$ is mean zero and let
$\Sigma := \E[xx^{\top}]$ denote the covariance matrix of the features. Without loss of generality, we will assume that the covariance matrix is diagonal and its eigenvalues are arranged in descending order $\lambda_1\ge \lambda_2\ge\ldots\ge\lambda_p> 0$. (Note that such a covariance matrix can always be obtained by a rotation and permutation, and the estimator~\eqref{def:two_layer_implicit_bias} is correspondingly transformed.)
Recall that $\by= (y_1,\ldots,y_n)^{\top}$ is the vector  of responses and $X= (x_1,\ldots,x_n)^{\top}$ is the data matrix. Define $\beps = (y_1-x_1^{\top}\theta^\star,\ldots,y_n-x_n^{\top}\theta^\star)^\top = (\epsilon_1,\ldots,\epsilon_n)^\top$ to be the vector of noise. 

We make the following assumptions:
\begin{enumerate}[({A}.1)]
\item \label{assumption:first}the samples $(x_1,y_1),\ldots,(x_n,y_n)$
and $(x,y)$
are drawn i.i.d.;
    \item the features $x$ and responses $y$ are mean-zero;
    \item \label{assumption:third}the features $x= \Sigma^{1/2}u$, where $u$ has components that are independent $\sigma_x^2$-sub-Gaussian random variables with $\sigma_x$ a positive constant, that is, for all $\phi \in \R^p$
    \begin{align*}
        \E\left[\exp\left(\phi^{\top}u\right)\right]\le \exp\left(\sigma_x^2\lv \phi\rv^2/2\right);
    \end{align*}
    \item \label{assumption:small_ball_probability} %NC
    there is an absolute constant c such that, for any unit vector $\phi \in \S^{n-1}$ and any $a\le b \in \R$ 
    \begin{align*}
        \Pr\left[(\Sigma^{-1/2}X^{\top}\phi)_i \in [a,b]\right] \le c|b-a|
    \end{align*}
    for all $i \in [p]$;
    \item \label{assumption:fifth}the difference $y-x^{\top}\theta^\star$ is $\sigma^2_y$-sub-Gaussian, conditionally on $x$, with $\sigma_y$ a positive constant, that is, for all $\phi \in \R$
    \begin{align*}
        \E_{y}\left[\exp\left(\phi(y-x^{\top}\theta^\star)\right)\; \big| \;x\right]\le \exp\left(\sigma_y^2 \phi^2/2\right)
    \end{align*}
    (note that this implies that $\E\left[y \mid x\right] = x^{\top}\theta^\star$);
    \item \label{assumption:sixth}for all $x$, the conditional variance
    of $y-x^{\top}\theta^\star$ is
    \begin{align*}
        \E_{y}\left[(y-x^{\top}\theta^\star)^2 \; \big| \; x\right] =\sigma^2
    \end{align*}
    where $\sigma$ is a positive constant.
\end{enumerate}
We emphasize
that $\sigma_x,\sigma_y$ and $\sigma$ are absolute constants, independent of all other problem parameters ($n,p$ and $\Sigma$). All the constants going forward may depend on the value of these constants. 

The assumptions stated above are satisfied in the case where $u$ is generated from a mean-zero isotropic log-concave distribution with sub-Gaussian, independent entries and the noise $y-x^{\top}\theta^\star$ is independent and sub-Gaussian. 
%NC
We note that Assumptions~\ref{assumption:first}-\ref{assumption:third}, \ref{assumption:fifth}-\ref{assumption:sixth} are standard in the literature of benign overfitting in linear models~\citep[see, e.g.,][]{bartlett2020benign}. We make an additional small-ball probability assumption (Assumption~\ref{assumption:small_ball_probability}) which allows us to derive a sharper multiplicative lower tail bound for the minimum eigenvalue of the submatrices of $X$ (Lemma~\ref{l:smallest_singular_value}).

Given the training samples define the excess risk of an estimate $\theta \in \R^p$ to be
\begin{align*}
    \mathsf{Risk}(\theta) := \E_{x,y}\left[(y-x^{\top}\theta)^2-(y-x^{\top}\theta^\star)^2\right],
\end{align*}
where $x,y$ are independent test samples.

Define the shorthand$$w:= \frac{\theta(0)}{\sqrt{\lv \theta(0)\rv}},$$ 
so that
the estimator described in equation~\eqref{def:two_layer_implicit_bias} can be written as the solution to a constrained convex program given by
\begin{align}
     \htheta &\in \argmin_{\theta \in \R^{p}} \; \lv \theta \rv^{3/2} - w^{\top} \theta, \qquad \text{s.t., } \; \by=X\theta. \label{e:estimator_def_in_v}
\end{align}

We let $UDV^{\top} = X$ be the singular value decomposition 
of $X$ where $U \in \R^{n\times n}$ and $V \in \R^{p\times p}$ are unitary matrices and $D \in \R^{n\times p}$ is a rectangular diagonal matrix with its eigenvalues in descending order. By 
Lemma~\ref{l:full_rank},
we know that the rank of $D$ is $n$. We let $D^{\dagger} \in \R^{p \times n}$ denote the pseudo-inverse of $D$. Since $D$ has rank $n$, the bottom $p-n$ rows of $D^{\dagger}$ are identically zero. 

Also define
\begin{align}
    \tby  := D^{\dagger}U^{\top}\by, \qquad
    \tw  := V^{\top} w \qquad \text{and}\qquad
    \ttheta : = V^{\top}\htheta. \label{e:modifed_data_defs}
\end{align}
We will use the following definitions of the
``effective rank'' from \cite{bartlett2020benign}.
\begin{definition}
Given a subset $S \subseteq [p]$,
define $s(S) := \sum_{i \in S}\lambda_i$,
and define
the following ranks of the covariance matrix $\Sigma$
with eigenvalues $\lambda_1,\ldots,\lambda_p$:
\begin{align*}
    % r(S) := \frac{\sum_{i \in S}\lambda_i}{\max_{i \in S}\lambda_i} \qquad \text{and} \qquad R(S) := \frac{(\sum_{i \in S}\lambda_i)^2}{\sum_{i \in S}\lambda_i^2}. 
    r(S) := \frac{s(S)}{\max_{i \in S}\lambda_i} \qquad \text{and} \qquad R(S) := \frac{s(S)^2}{\sum_{i \in S}\lambda_i^2}. 
\end{align*}
Further given any $j \in [p]$,
with some abuse of notation, define $s_j := \sum_{i>j}\lambda_i$ and 
\begin{align*}
    r_j := \frac{s_j}{\lambda_{j+1}} \qquad \text{and} \qquad R_j := \frac{s_j^2}{\sum_{i>j}\lambda_i^2}.
\end{align*}
\end{definition}
The following lemma~\citep[][Lemma~5]{bartlett2020benign} relates these different effective ranks.
\begin{lemma}\label{l:effective_ranks}For any subset $S \subseteq [p]$ the ranks defined above satisfy the following:
\begin{align*}
    r(S) \le R(S) \le r(S)^2.
\end{align*}
\end{lemma}

We define the index $k$ below. The value of $k$ shall help determine what we consider the ``tail'' of the covariance matrix. 
\begin{definition}\label{def:k_star}
For a large enough constant $b$ (that will be fixed henceforth), define
\begin{align*}
    k := \min\{j\ge 0: r_j\ge b n\},
\end{align*}
where the minimum of the empty set is defined as $\infty$.
\end{definition}

 Finally we define $\psi$, which is a rescaling of $w$.
\begin{definition}
\label{d:psi}
Define
\begin{align*}
    \psi &:= \frac{2 \sqrt{\sigma}n^{1/4}}{3s_k^{1/4}}w = \frac{2  \sqrt{\sigma}n^{1/4}}{3 s_k^{1/4}} \frac{\theta(0)}{\sqrt{\lv \theta(0)\rv}}.
\end{align*}
\end{definition}
\section{Main results}\label{s:main_results}

In this section we present our main result, Theorem~\ref{t:main}, which is an excess risk bound for the estimator $\htheta$. It is proved in Section~\ref{s:proof_details}.
\begin{restatable}{theorem}{main}
\label{t:main}
Under Assumptions \ref{assumption:first}-\ref{assumption:sixth}, 
there exist constants $c_0,\ldots,c_7$ such that for any $\delta \in (e^{-c_0\sqrt{n}},1-c_1e^{-c_2n})$, 
if $p\ge c_3(n+k)$, $n\ge c_4 \max\left\{k,s_k\right\}$ and $\lv\theta^\star\rv,\lv w\rv \le c_5$ then with probability at least $1-c_6 \delta$ 
\begin{align*}
    \mathsf{Risk}(\htheta)
 &\leq \mathsf{Bias}+\mathsf{Variance}+\mathsf{\Xi},
 \end{align*}
 where
 \begin{align*}
     \mathsf{Bias}&\le c_7\left(\lv(\theta^\star-\psi)_{1:k} \rv_{\Sigma^{-1}_{1:k}}^2\left(\frac{s_k}{n}\right)^2+\lv (\theta^\star-\psi)_{k+1:p}\rv_{\Sigma_{k+1:p}}^2  \right)\le \frac{2c_7 \lv\theta^\star -\psi \rv^2 s_k}{n};\\
     \mathsf{Variance} &\le c_7\log(1/\delta)\left(\frac{k}{n}+\frac{n}{R_{k}}\right);\\
     \mathsf{\Xi} &\le c_{7}\lambda_1\lv \psi\rv^2 \left[\frac{n}{R_{k}}+\frac{n^2}{r_{k}^2}+\frac{s_k}{n}+ \frac{\log(1/\delta)}{n}+\frac{ k^2}{n^2} \right]\max\left\{\sqrt{\frac{r_0}{n}},\frac{r_0}{n},\sqrt{\frac{\log(1/\delta)}{n}}\right\}.
 \end{align*}
\end{restatable}
Note that $\mathsf{Bias}$ goes to zero as $\psi \to \theta^\star$. In the upper bounds on the excess risk for linear models with a standard (one-layer) parameterization
\citep[see][Theorem~1]{tsigler2020benign},
the corresponding term scales with the square of the norm of $\theta^\star$ rather than $(\theta^\star-\psi)$. 
%NC edits here to clarify stuff about psi and theta(0)
If one has a ``guess'' $\widehat{\psi}$ for $\theta^\star$, then---given knowledge of $\sigma, s_k$ and $n$---it is possible to set the initialization as follows:
\begin{align*}
    \theta(0)= \frac{9 \widehat{\psi} \lv \widehat{\psi}\rv}{4}\sqrt{\frac{s_k}{\sigma^2 n}};
\end{align*}
which ensures that $\psi = \widehat{\psi}$. Very accurate prior guesses $\widehat{\psi}$ of $\theta^*$ are rewarded with a
very small value of the $\mathsf{Bias}$ term. 
% Viewing
% $\psi$ as a ``guess'' for $\theta^\star$, then---given knowledge of $\sigma, s_k$ and $n$---it is possible to set the initialization as follows:
% \begin{align*}
%     \theta(0)= \frac{9 \psi \lv \psi\rv}{4}\sqrt{\frac{s_k}{\sigma^2 n}};
% \end{align*}
% very accurate prior guesses $\psi$ of $\theta^*$ are rewarded with a
% very small value of the $\mathsf{Bias}$ term. 

%NC 
Next, we note that the upper bound on $\mathsf{Variance}$ here is identical to the upper bound on the variance for the minimum $\ell_2$-norm interpolant (the OLS estimator)~\citep[see][Theorem~4]{bartlett2020benign}. The initialization $\theta(0)$ (through $\psi$) only affects the conditional bias of the estimator here, but leaves the conditional variance the same as the OLS solution. This is because, as we will show below in Lemma~\ref{l:theta_solution}, $\htheta$ can be expressed as a perturbation to the OLS estimator in the subspace orthogonal to the row span of $X$. It turns out that the variance only depends on behavior of $\htheta$ in the subspace spanned by the data, where $\htheta$ and the OLS solution are identical.

As mentioned, $\htheta$ is a perturbation of the OLS estimator. In particular, it is perturbed by $\alpha^\star \mathsf{Proj}_{X}^{\perp}(w)$, where $\mathsf{Proj}_{X}^{\perp}(w)$ is the projection of $w$ onto the subspace 
orthogonal to the row span of $X$
and $\alpha^\star$ is a scalar random variable that depends on the data. We shall demonstrate in Lemma~\ref{l:bound_on_alpha} that, under the setting specified by the theorem, $\alpha^\star$ concentrates around $\frac{2\sqrt{\sigma}n^{1/4}}{3s_k^{1/4}}$. The final term in the excess risk bound, $\mathsf{\Xi}$, corresponds to the 
fluctuation of 
$\alpha^\star$. We might think of $\theta(0)$ (and hence $w$) as being constructed from $\psi$ and an estimate of $\alpha^*$; from this point of
view, $\mathsf{\Xi}$ accounts for the contribution to the excess risk arising
from the error in estimating $\alpha^\star$.
%NC updates to the discussion below
Next we derive sufficient conditions for the excess risk to go to zero as $n, p \to \infty$. Consider the case where $\lambda_1$, $\lv \theta^\star\rv$, $\lv \psi \rv$  and $\log(1/\delta)$ are all bounded by constants. 
(In the case of $\lv \psi \rv$, this
can be achieved by appropriately scaling $\theta(0)$.)
For $\mathsf{Bias}$ to go to zero it suffices if
\begin{align*}
    \frac{s_k}{n} \to 0.
\end{align*}
For $\mathsf{Variance}$ to decrease to zero it suffices if 
\begin{align*}
    \frac{k}{n} \to 0 \quad \text{and} \quad \frac{n}{R_{k}}\to 0.
\end{align*}
Finally, for $\mathsf{\Xi}$ to approach zero it suffices
for
\begin{align*}
    \frac{r_0}{n}\to 0
\end{align*}
which also implies the condition 
$\frac{s_k}{n} \rightarrow 0$ needed to
control the $\mathsf{Bias}$ term.
(To see that $\frac{r_0}{n}\to 0$ suffices, recall that
we have assumed that $\lambda_1, \log(1/\delta)$ and
$\lv \psi \rv$ are constants.  Further, the
quantity in the square brackets of our bound on
$\mathsf{\Xi}$ is at most a constant, which can be seen
as follows.  The definition of $r_k$ implies that
$r_k \geq b n$, and Lemma~\ref{l:effective_ranks} gives
$R_k \geq r_k$.  Finally, we have assumed that
$n \geq c \max \{ k, s_k\}$.)
To summarize, 
if $\frac{k}{n}, \frac{r_0}{n}, \frac{n}{R_k} \to 0$, 
the excess risk of this estimator 
approaches zero.
Some discussion and examples
of when this condition is satisfied are given
in~\citep{bartlett2020benign,tsigler2020benign}.

To develop intuition, we consider a special case
of Theorem~\ref{t:main} defined as follows.
\begin{definition}[$(k,\epsilon)$-spike model]
For $\epsilon > 0$ and $k \in \N$, 
a  $(k,\epsilon)$-spike model is 
a setting
where the eigenvalues of $\Sigma$ are
$\lambda_1 = \ldots = \lambda_k=1$ and
$\lambda_{k+1} = \ldots = \lambda_p = \epsilon$.  
\end{definition}

Instantiating Theorem~\ref{t:main} in the case of the $(k,\epsilon)$-spike model, and removing some
dominated terms, yields the following corollary.
\begin{corollary}
\label{c:spike}
Under Assumptions \ref{assumption:first}-\ref{assumption:sixth}, 
there exist constants $c_0,\ldots,c_8$ such that in the $(k,\epsilon)$-spike model
for any $\delta \in (e^{-c_0\sqrt{n}},1-c_1e^{-c_2n})$, if 
 $p>c_3(n+k)$, $n\ge c_4 \max\left\{k,\epsilon p\right\}$ and $\lv\theta^\star\rv,\lv w\rv \le c_5$ then with probability at least $1-c_6 \delta$ 
\begin{align*}
    \mathsf{Risk}(\htheta)
 &\leq \mathsf{Bias}+\mathsf{Variance}+\mathsf{\Xi},
 \end{align*}
 where
 \begin{align*}
     \mathsf{Bias}&\le c_7\left(\lv(\theta^\star-\psi)_{1:k} \rv^2\left(\frac{\epsilon p}{n}\right)^2+\epsilon\lv (\theta^\star-\psi)_{k+1:p}\rv^2  \right)\le c_8 \lv \theta^\star-\psi\rv^2 \left(\frac{ \epsilon p }{n}\right);\\
     \mathsf{Variance} &\le c_7\log(1/\delta)\left(\frac{k}{n}+\frac{n}{p}\right);\\
     \mathsf{\Xi} &\le c_{7}\lambda_1
      \lv \psi\rv^2 
      \left[\frac{n}{p}+\frac{\epsilon p}{n}+ \frac{\log(1/\delta)}{n}+\frac{ k^2}{n^2} \right]\max\left\{\sqrt{\frac{k+\epsilon p}{n}},\sqrt{\frac{\log(1/\delta)}{n}}\right\}.
 \end{align*}
\end{corollary}
Again, in the case where $\lambda_1,\lv \psi\rv$ and $\log(1/\delta)$ are bounded by constants, a sufficient condition for the excess risk to decrease to zero is when $\frac{\epsilon p }{n},\frac{k}{n},\frac{n}{p}\to 0$.

Next we establish a lower bound. It is proved in Section~\ref{s:lower}.
\begin{restatable}{proposition}{lowerbound}
\label{p:lower}
If $a(0)$ and $W(0)$ are chosen randomly, independent of
$X$ and $\by$, so that the distribution of 
$a(0)^{\top} W(0)$ is symmetric about the origin, then 
\begin{align*}
\E_{a(0), W(0), X, \by} [ \mathsf{Risk}(\htheta)]
 &\geq 
       \E\left[\theta^{\star\top} B \theta^\star \right]
               + \sigma^2 \E\left[\Tr(C)\right],
\end{align*}
where
\begin{align*}
    B &:= \left(I - X^{\top}(XX^{\top})^{-1}X\right)\Sigma\left(I - X^{\top}(XX^{\top})^{-1}X\right) \quad \text{and} \\
    C &:= (XX^{\top})^{-1}X\Sigma X^{\top}(XX^{\top})^{-1}.
\end{align*}
\end{restatable}

\begin{remark}
For the distribution of 
$a(0)^{\top}W(0)$ to be symmetric about the
origin, it suffices that $a(0)$ and $W(0)$
are chosen independently, and that either the distribution
of $a(0)$ is symmetric about the origin, or
the distribution of $W(0)$ is.
\end{remark}

\begin{remark}
\citet{bartlett2020benign} proved that
$\E[\Tr(C)] \geq c \left( 
  \frac{k}{n} + \frac{n}{R_k}
   \right)$ for a constant $c$.
\citet{tsigler2020benign} proved
a lower bound on $\E[\theta^{\star \top} B \theta^\star ]$
under the assumption that the
signs of the components of $\theta^{\star}$ are
chosen uniformly at random.
For the case that $\psi = 0$, 
their lower bound
matches the upper bound on
$\mathsf{Bias}$ from Theorem~\ref{t:main} of this paper
under the assumptions of that theorem. 
%NC what do you think of this?
%PL Looks good to me!
However, there is a gap in the upper and lower bounds when $\psi \neq 0$.
\end{remark}

\section{Proof details}\label{s:proof_details}
The proof of Theorem~\ref{t:main} is built up in parts. First, in Lemma~\ref{l:theta_solution} we show that $\htheta$ can be viewed as a random perturbation of the ordinary least squares (OLS) solution in the subspace orthogonal to the row span of $X$. In Lemma~\ref{l:excess_risk_decomposition}, we show that the excess risk can be decomposed into two terms, one that can 
% upper bounded 
bounded above
by $\mathsf{Variance}$ and the other that is upper bounded by $\mathsf{Bias}+\mathsf{\Xi}$. The next piece is Lemma~\ref{l:bound_on_alpha} which is crucial in helping us characterize the perturbation to the OLS solution. To do this we first present concentration inequalities in Section~\ref{ss:largest_singular_value}, then we establish upper and lower bounds on $\Tr((XX^{\top})^{-1})$ in Section~\ref{ss:concentration_of_Tr_A_inverse}, and finally prove Lemma~\ref{l:bound_on_alpha} in Section~\ref{s:alpha_bound}. We finish by combining all of these elements to prove the theorem in Section~\ref{ss:main_theorem_proof}. Throughout this section we assume that the assumptions made in Theorem~\ref{t:main} are in force.

\noindent\paragraph{A note about constants.}
% These are absolute constants.  
As mentioned earlier, we will not always provide specific constants.
The constants $c_1, c_2, \ldots$ in our proofs are independent of the problem parameters, but they can depend on
one another.  It will not be hard to verify, however, that the constraints on their values
are satisfiable. When we write ``$c_i$ is large enough'', this should be understood to be relative 
to the constants previously introduced in the proof not including $b$, the constant used in the definition of $r_k$. Loosely speaking, $b$ is chosen last: it should be taken to be large relative
to all other constants.

We begin with the following lemma that provides a closed-form formula for $\htheta$ as a perturbation of the ordinary least squares solution.
\begin{lemma}\label{l:theta_solution}
The solution $\htheta$ can be expressed as follows:
\begin{align*}
    \htheta = \htheta_{\mathsf{OLS}}+\alpha^\star (I-X^{\top}(XX^{\top})^{-1}X)w,
\end{align*}
where $    \htheta_{\mathsf{OLS}} = X^{\top}\left(XX^{\top}\right)^{-1}\by$ is the ordinary least squares solution (minimum $\ell_2$-norm interpolant) and 
\begin{align}
    \alpha^\star & = \sqrt{\frac{8\lv\tw_{n+1:p}\rv^2+\sqrt{64\lv\tw_{n+1:p}\rv^4+1296\lv\tby\rv^2}}{81}}. \label{def:alpha_star}
\end{align}
where $\tw$ and $\tby$ are defined in \eqref{e:modifed_data_defs}.
\end{lemma}
\begin{proof}
Recall that $X = UDV^{\top}$ is the singular value decomposition of $X$. Therefore,
    \begin{align*}
    &\htheta \in \argmin_{\theta \in \R^{p}} \; \lv \theta \rv^{3/2} - w^{\top} \theta, \qquad \text{s.t., } \; \by=X\theta \\
    \iff & \htheta \in \argmin_{\theta \in \R^{p}} \; \lv \theta \rv^{3/2} - w^{\top} \theta, \qquad \text{s.t., } \; D^{\dagger}U^{\top}\by=D^{\dagger}U^{\top}X\theta \\
    \iff & \htheta \in \argmin_{\theta \in \R^{p}} \; \lv \theta \rv^{3/2} - w^{\top} \theta, \qquad \text{s.t., } \; \tby=D^{\dagger}U^{\top}X\theta \\
    \iff & V^{\top}\htheta \in \argmin_{\theta \in \R^{p}} \; \lv V^{\top}\theta \rv^{3/2} - (w^{\top}V)V^{\top} \theta, \qquad \text{s.t., } \; \tby=D^{\dagger}U^{\top}XV(V^{\top}\theta) \\
    & \hspace{1in} \mbox{(since the Euclidean norm is rotation invariant and $VV^{\top} = I_p$)}\\
    \iff &\ttheta \in \argmin_{\theta \in \R^{p}} \; \lv \theta \rv^{3/2} - \tw^{\top}\theta, \qquad \text{s.t., }\;  \tby=D^{\dagger}U^{\top}XV \theta.
\end{align*}
Since the bottom $p-n$ rows of $D^{\dagger}$ are identically zero, the vector $\tby$ can be written $(\tby_1,\ldots,\tby_n,0,\ldots,0)^{\top}$. 
Since the SVD of $X= UDV^{\top}$ we have that
\begin{align*}
    D^{\dagger}U^{\top}XV\theta & = D^{\dagger}U^{\top}UDV^{\top}V\theta  = D^{\dagger} D \theta = \begin{bmatrix} I_n & 0\\
    0 & 0
    \end{bmatrix}\theta.
\end{align*}
Hence, for the constraint to be satisfied, the first $n$ coordinates of $\ttheta$ are required to be equal to $(\tby_1,\ldots,\tby_n)$, and the remaining coordinates
of $\ttheta$ can be anything.
%Suppose that 
That is, the constraints are satisfied when
$\ttheta = (\tby_{1},\ldots,\tby_{n}, 0,\ldots,0)^\top+\phi$, for some
$\phi \in \R^{p}$ with its first $n$ coordinates all equal to zero. 

To find this optimal vector $\phi^\star$ we can now proceed to solve the following \emph{unconstrained} optimization problem:
\begin{align*}
    \phi^\star \in \argmin_{\phi \in \R^{p}} \left(\lv \ttheta_{1:n}\rv^2+\lv \phi\rv^2\right)^{3/4}-\sum_{j=n+1}^{p}\tw_{j}\phi_j.
\end{align*}
Since the first term in the objective function above is rotationally invariant, it must be the case that the minimizer has the form $\phi^\star = \alpha^\star \tw_{n+1:p}$, for some $\alpha^\star >0$. That is, it is positively aligned with the tail of the vector $\tw$. 
(If $\phi^\star$ 
% lied outside 
was not in
the span of $\tw_{n+1:p}$, removing the projection
of $\phi^\star$ in the subspace orthogonal to this direction would
improve the norm without affecting the second term, and if
$\phi^\star \cdot \tw_{n+1:p} < 0$, then $-\phi^\star$ would be a better
solution than $\phi^\star$.)  In particular, we have
\begin{align*}
\label{e:nothing.outside.span}
\tw_{n+1:p} = 0 \Rightarrow \phi^\star = 0.
\end{align*}
Otherwise,
% we can instead solve 
$\phi^\star = \alpha^\star \tw_{n+1:p}$ for the solution $\alpha^\star$
of the following one-dimensional problem:
\begin{align*}
    \alpha^\star \in \argmin_{\alpha >0} \; \left(\lv \ttheta_{1:n}\rv^2 + \alpha^2 \lv \tw_{n+1:p}\rv^2\right)^{3/4} - \alpha \lv \tw_{n+1:p}\rv^2.
\end{align*}
To simplify notation, let $\rho:=\lv \ttheta_{1:n}\rv^2$ and $\zeta:=\lv \tw_{n+1:p}\rv^2 > 0$. The first derivative of the objective function is as follows:
\begin{align*}
    \frac{d}{d\alpha}\left[\left(\rho + \zeta \alpha^2 \right)^{3/4} - \alpha \zeta\right] &= \frac{3\zeta\alpha }{2\left(\rho+\zeta\alpha^2\right)^{1/4}} -\zeta.
\end{align*}
Setting this first derivative equal to zero we get that
\begin{align*}
    &\frac{3\zeta\alpha }{2\left(\rho+\zeta\alpha^2\right)^{1/4}} -\zeta = 0 \\
    \iff & \frac{3\alpha }{2\left(\rho+\zeta\alpha^2\right)^{1/4}} -1 = 0 \\
    \iff & \frac{3\alpha }{2}  = \left(\rho+\zeta\alpha^2\right)^{1/4}\\
    \iff & \frac{81\alpha^4 }{16}  = \rho+\zeta\alpha^2 \hspace{0.5in} \mbox{(because $\alpha > 0$ at the optimum)} \\
    \iff & 81 \alpha^4 -16\zeta\alpha^2 -16\rho = 0.
\end{align*}
We can view this as a quadratic equation in $\alpha^2$, so
\begin{align*}
    \alpha^2 = \frac{16\zeta \pm \sqrt{256\zeta^2 + 5184 \rho}}{162} = \frac{16\zeta\left(1 \pm \sqrt{1 + \frac{81 \rho}{4 \zeta^2}}\right)}{162},
\end{align*}
but the solution with the negative sign can be ignored since $\alpha^2$ must be positive. Taking square roots we get that
\begin{align*}
    \alpha = \pm \sqrt{ \frac{8\zeta\left(1 + \sqrt{1 + \frac{81 \rho}{4
    \zeta^2}}\right)}{81}}.
\end{align*}
Again, we drop the negative solution since we know that $\alpha > 0$ at the optimum.
Thus we find that
\begin{align*}
    \ttheta 
     & = \tby + \alpha^\star
      \tw_{n+1:p}
\end{align*}
for 
\begin{align*}
    \alpha^{\star} 
       = \sqrt{ \frac{8\zeta\left(1 + \sqrt{1 + \frac{81 \rho}{4 \zeta^2}}\right)}{81}} 
      & = \sqrt{ \frac{8 \lv \tw_{n+1:p}\rv^2 
                          \left(1 + 
                                \sqrt{1 + \frac{81 \lv \ttheta_{1:n}\rv^2}{4 \lv
                                \tw_{n+1:p}\rv^4}}\right)}
                      {81}} \\
      & = \sqrt{ \frac{8 \lv \tw_{n+1:p}\rv^2 
      + \sqrt{64 \lv \tw_{n+1:p}\rv^4 + 1296 \lv \ttheta_{1:n}\rv^2 }}{81}} \\
      & = \sqrt{ \frac{8 \lv \tw_{n+1:p}\rv^2 
      + \sqrt{64 \lv \tw_{n+1:p}\rv^4 + 1296 \lv \tby\rv^2 }}{81}}.
\end{align*}

Recall that by definition $\ttheta = V^{\top}\htheta$, $\tby = D^{\dagger}U^{\top}\by$ and $\tw = V^{\top} w$ and hence
\begin{align*}
    \htheta &= V\tby+ \alpha^\star V\tw_{n+1:p}\\ &= VD^{\dagger} U^{\top} \by + \alpha^\star V\tw_{n+1:p}\\ 
     &= X^{\top}\left(XX^{\top}\right)^{-1} \by + \alpha^\star V\tw_{n+1:p}\\
     &= \htheta_{\mathsf{OLS}} + \alpha^\star V\tw_{n+1:p}\\
     & = \htheta_{\mathsf{OLS}} + \alpha^\star V\begin{bmatrix}0_{n\times n} & 0_{n\times(p-n)}\\ 0_{(p-n)\times n} & I_{p-n}\end{bmatrix}V^{\top}w.
\end{align*}
Recall that the SVD of $X$ is  $UDV^{\top}$, and the last $(p-n)$ 
columns of $D$ are 
zero.  Thus, the last $(p-n)$ rows of $V^\top$ span the null
space of $X$.  

Furthermore, $(I_p - X^{\top}(XX^{\top})^{-1}X )$
represents the projection onto this null space of $X$.
This can be seen as follows.
First, any member $u$ of
this null space is mapped to itself (since $X u = 0$).
On the other hand, for each row $x$ of $X$, 
$(I_p - X^{\top}(XX^{\top})^{-1}X ) x^{\top} = 0$,
as 
\[
(I_p - X^{\top}(XX^{\top})^{-1}X ) X^{\top} = 0.
\]

Recalling that the last $(p-n)$ rows of
$V^\top$ span the null
space of $X$, we have
\begin{align*}
    V\begin{bmatrix}0_{n\times n} & 0_{n\times(p-n)}\\ 0_{(p-n)\times n} & I_{p-n}\end{bmatrix}V^{\top} = I_p - X^{\top}(XX^{\top})^{-1}X .
\end{align*}
This wraps up our proof.
\end{proof}

Armed with this formula for $\htheta$ we can now bound the excess risk.
% from above.
\begin{lemma}
\label{l:BC}
The excess risk of $\htheta$ satisfies
\begin{align*}
    \mathsf{Risk}(\htheta) & \le c (\theta^\star - \alpha^\star w)^{\top}B(\theta^\star - \alpha^\star w)+c\log(1/\delta)\Tr(C)
\end{align*}
with probability at least $1-\delta$ over $\beps$, where
\begin{align*}
    B &:= \left(I - X^{\top}(XX^{\top})^{-1}X\right)\Sigma\left(I - X^{\top}(XX^{\top})^{-1}X\right) \quad \text{and} \\
    C &:= (XX^{\top})^{-1}X\Sigma X^{\top}(XX^{\top})^{-1}.
\end{align*}
\label{l:excess_risk_decomposition}
\end{lemma}
\begin{proof}
Since $\epsilon = y - x^{\top}\theta$ is conditionally mean-zero given $x$,
\begin{align*}
    \mathsf{Risk}(\htheta) &= \E_{x,y}\left[(y - x^{\top}\htheta)^2\right]-\E_{x,y}\left[(y - x^{\top}\theta^{\star})^2\right]\\
    &= \E_{x,y}\left[(y - x^{\top}\theta^{\star}+x^{\top}(\theta^\star -\htheta))^2\right]-\E_{x,y}\left[(y - x^{\top}\theta^{\star})^2\right]\\
    & = \E_{x}\left[\left(x^{\top}(\theta^\star -\htheta)\right)^2\right].
\end{align*}
Using the formula of $\htheta$ from Lemma~\ref{l:theta_solution}, and because $\by = X\theta^\star +\beps$ we find that
\begin{align*}
    \mathsf{Risk}(\htheta) & = \E_x\left[\left(x^{\top}\left(I-X^{\top}(XX^{\top})^{-1}X\right)(\theta^\star -\alpha^\star w)-x^\top X^\top(XX^\top)^{-1}\beps\right)^2\right]\\
    & \le 2\E_x\left[\left(x^{\top}\left(I-X^{\top}(XX^{\top})^{-1}X\right)(\theta^\star -\alpha^\star w)\right)^2\right]+2\E_{x}\left[\left(x^\top X^\top(XX^\top)^{-1}\beps\right)^2\right]\\
    & = 2(\theta^\star -\alpha^\star w)^{\top} B(\theta^\star -\alpha^\star w)+2\beps^{\top} C \beps.
\end{align*}
Now by \citep[][Lemma~19]{bartlett2020benign} we find that $2\beps^{\top} C \beps \le c'\sigma^2_y\log(1/\delta)\Tr(C)\le c \log(1/\delta)\Tr(C)$ with probability at least $1-\delta$. This completes the proof.
\end{proof}

The following lemma provides upper and lower bounds on the value of $\alpha^\star$ that are tight up to the leading constant when $p$ is large relative to $n+k$ and when $n$ is sufficiently large relative to $k$ and $s_k = \sum_{j>k}\lambda_j$.
\begin{restatable}{lem}{boundonalpha}
\label{l:bound_on_alpha}There are constants $c_0,\ldots,c_5$ such that for any $\delta \in (e^{-c_0 \sqrt{n}},1)$, if $p\ge c_1 (n+k)$, $n\ge c_2 \max\left\{k,s_k\right\}$ and $\lv \theta^\star \rv , \lv w \rv \le c_3$ then with probability at least $1-c_4\delta$,
\begin{align*}
    \left|\frac{\alpha^\star}{\frac{2\sqrt{\sigma}n^{1/4}}{3s_k^{1/4}}} -1\right|\le c_5\left[\sqrt{\frac{n}{R_{k}}}+\frac{n}{r_{k}}+\sqrt{\frac{s_k}{n}}+ \sqrt{\frac{\log(2/\delta)}{n}}+\frac{ k}{n} \right].
\end{align*}
\end{restatable}
Lemma~\ref{l:bound_on_alpha} is proved over the next few subsections.
As might be expected, for $\alpha^\star$ to reliably fall within a
small interval, $X$ must be well conditioned in some sense.
We begin by establishing bounds on the singular values of submatrices of
$X$ in Sections~\ref{ss:largest_singular_value},
whose proofs are
provided Appendix~\ref{a:concentration}.
In Section~\ref{ss:concentration_of_Tr_A_inverse}, we show
that the $\Tr((X X^\top)^{-1})$ is concentrated.
Armed with these bounds, we
build up our analysis of $\alpha^\star$ in stages in Section~\ref{s:alpha_bound}. 

\subsection{Bounds on the extreme singular values of submatrices of $X$}
\label{ss:largest_singular_value}
In this subsection, we will derive bounds on the largest and smallest 
singular value
of a submatrix of $X$. Given a subset $S$ of $[p]$, let $X_{S} \in \R^{n\times |S|}$ be a submatrix of $X$ where only the columns with indices in $S$ are included. 

With this in place, we are now ready to prove our concentration results. We prove this lemma in Appendix~\ref{s:eigenvalue_bound}.
\begin{restatable}{lem}{eigenvaluebound}
\label{l:eigenvalue_bound}
There exists a positive absolute constant $c$ such that, for any subset $S \subseteq [p]$ and any $t\ge 0$, with probability at least  $1-2e^{-t}$, for all $j \in \{1,\ldots,\min(n,|S|)\}$
\begin{align*}
% \left|\mu_j(X_SX_S^{\top})-r_{S}\max_{i \in S}\lambda_i\right|&\le 
%      cr_{S}\max_{i \in S}\lambda_i\left(\frac{t+n}{r(S)}+\sqrt{\frac{t+n}{R(S)}}\right).
     \left|\mu_j(X_SX_S^{\top})- s(S) \right|&\le 
     cs(S) \left(\frac{t+n}{r(S)}+\sqrt{\frac{t+n}{R(S)}}\right).
\end{align*}
\end{restatable}
This lemma provides an additive lower bound on the minimum singular value of submatrices of $X$. Next, we will provide a sharper multiplicative bound on the smallest singular value of such matrices. Its proof can be found in Appendix~\ref{s:smallest_multiplicative_bound_appendix}.
%NC change to an assumption in the lemma
\begin{restatable}{lem}{smallestsingularvalue} \label{l:smallest_singular_value}There exist absolute positive constants $c_0,\ldots,c_3$ such that given any subset $S \subseteq [p]$ if, 
% $|S|\ge c_0n$
%NC
$r(S)\ge c_0n$ 
then for all $t<c_1<1$ 
\begin{align*}
        \Pr\left[\mu_{n}(X_SX_{S}^{\top}) \le t \cdot s(S) \right] \le (c_2 t)^{c_3 \cdot r(S)}.
\end{align*}
\end{restatable}
%NC 
This sharper multiplicative bound provides a much more refined lower tail probability estimate for the minimum eigenvalue than the previous additive bound in Lemma~\ref{l:eigenvalue_bound}, especially when $t$ is close to zero. This is useful in our analysis to control $\mathbb{E}[\Tr(XX^{\top})^{-1}]$ which is in turn used to establish Lemma~\ref{l:concentration_of_Tr_A_inverse} that bounds $\Tr((XX^{\top})^{-1})$.

\subsection{Concentration of $\Tr( (X X^{\top})^{-1})$}
\label{ss:concentration_of_Tr_A_inverse}

In this subsection we shall prove the following lemma which shows that $\Tr((XX^{\top})^{-1})$ concentrates.
\begin{restatable}{lem}{concentrationofTrAinverse}
\label{l:concentration_of_Tr_A_inverse} 
There are positive constants 
$c_0,\ldots,c_4$ such that, if $p \ge c_0(n+k)$ then with probability at least $1-c_1e^{-c_2n}$
\begin{align*}
    \left|\Tr\left((XX^\top)^{-1}\right)-\frac{n}{s_k}\right|& \le  \frac{c_3n}{s_k}\left[\sqrt{\frac{n}{R_{k}} }+\frac{n}{r_{k}}+\frac{k}{n}+e^{-c_4\sqrt{n}}\right].
\end{align*}
\end{restatable}

The proof of Lemma~\ref{l:concentration_of_Tr_A_inverse} in turn
requires some lemmas.  To state them,
we need some additional notation and definitions. 

Recall that we have assumed without loss of generality
that $\Sigma$ is diagonal.  Let $\lambda_1 \geq \ldots \geq \lambda_p$
be the elements of its diagonal,
% Let $\Sigma = \sum_{i=1}^p \lambda_i e_ie_i^{\top}$ be
% the spectral decomposition of the covariance matrix 
and define the random vectors $$z_i := \frac{Xe_i}{\sqrt{\lambda_i}} \in \R^n.$$ 
These random vectors $z_i$ have entries that are independent, $\sigma_x^2$-sub-Gaussian random variables~\citep[see][Lemma~8]{bartlett2020benign}. Note that we can write the matrix
\begin{align*}
    XX^{\top}  = \sum_{i=1}^p \lambda_i z_i z_i^{\top}.
\end{align*}
\begin{definition} \label{def:A_head_tail}
Define the shorthand $A:=XX^{\top}$, and define
\begin{align*}
    H := \sum_{i=1}^{k} \lambda_i z_i z_i^{\top} \qquad \text{and} \qquad
    T := \sum_{i=k+1}^p \lambda_i z_i z_i^{\top}.
\end{align*}
Therefore $A = H + T$.
\end{definition} 

To prove Lemma~\ref{l:concentration_of_Tr_A_inverse} we shall prove the following four results:
\begin{itemize}
    \item in Lemma~\ref{l:trace_A_close_to_trace_T}, we show that $\Tr(A^{-1})$ is close to the $\Tr(T^{-1})$ with high probability;
    \item in Lemma~\ref{l:expectations_are_close}, we show that $\E\left[\Tr(A^{-1})\right]$ is well approximated by $\E\left[\Tr(T^{-1})\right]$;
    \item in Lemma~\ref{l:trace_T_close_to_its_expectation}, we show that $\Tr(T^{-1})$ is close to its expectation with high probability;
    \item finally, in Lemma~\ref{l:upper_lower_bounds_on_expectation_of_Tr_A_inverse}, we establish upper and lower bounds on $\E\left[\Tr(A^{-1})\right]$ that match up to leading constants.
\end{itemize}
By using these four results and the triangle inequality we shall demonstrate that $\Tr(A^{-1})$ is close to $n/s_k$ with high probability and prove Lemma~\ref{l:concentration_of_Tr_A_inverse}. Throughout this subsection we shall assume that the dimension $p \ge c_0(n+k)$, for a sufficiently large constant $c_0$. Under this condition, Lemma~\ref{l:full_rank} implies that the tail matrix $T$ is full-rank and invertible.

\subsubsection{$\Tr(A^{-1})$ is close to $\Tr(T^{-1})$}
We begin by showing that $\Tr(A^{-1})$ is close to $\Tr(T^{-1})$ with high probability.
\begin{lemma}\label{l:trace_A_close_to_trace_T} 
There exist positive constants $c_0, \ldots,c_3$ such that, for all $\beta < c_0 <1$, with probability at least $1-2\exp(- r_{k}/\beta^2)-(c_1 \beta)^{c_2\cdot r_{k}}$,
\begin{align*}
    \left|\Tr(A^{-1})-\Tr(T^{-1})\right|& \le  \frac{c_3 k}{\beta^4s_k}.
\end{align*}
\end{lemma}
\begin{proof}Recall that $A=XX^{\top}=\sum_{i=1}^{k}\lambda_i z_iz_i^{\top}+\sum_{i=k+1}^p\lambda_i z_iz_i^{\top}=H+T$. Let $u_1,\ldots,u_{k} \in \R^{n}$ be an orthonormal
basis for the row span of
$H$, and let $u_1,\ldots,u_n$ 
be an extension to a basis for $\R^n$. Write $U=[u_1,\ldots, u_n]=[E; F]$, where $E$ is $n \times k$. Thus
\begin{align*}
\Tr(A^{-1})
 & \overset{(i)}{=} \Tr\left(U^\top A^{-1}U\right) \\
 & = \Tr\left(\left[U^\top AU\right]^{-1}\right) \\
 & = \Tr\left(\left[U^\top (H + T)U\right]^{-1}\right) \\
 & = \Tr\left(\left[U^{\top}HU + U^{\top}TU\right]^{-1}\right)\\
  & = \Tr\left(\left[\begin{pmatrix}E^\top H E & E^{\top}H F \\F^{\top} H E & F^{\top} H F\end{pmatrix}  +  U^\top T U\right]^{-1}\right) \\
  & \overset{(ii)}{=} \Tr\left(\left[\begin{pmatrix}E^\top H E & 0 \\0 & 0\end{pmatrix}  +  U^\top T U\right]^{-1}\right) 
  \\& = \Tr\left(\left[\begin{pmatrix}E^\top \\0\end{pmatrix} H \begin{pmatrix} E & 0\end{pmatrix}  + U^\top T U\right]^{-1}\right),
 \end{align*}
 where $(i)$ follows since $U$ is a unitary matrix, $(ii)$ follows since the columns of $F$ are outside the span of $H$. Continuing, we apply the Sherman-Morrison-Woodbury identity to get that
 \begin{align*}
 &\Tr(A^{-1})\\
     & = \Tr\left(\left[ U^{\top}TU\right]^{-1}\right)\\
&\quad  -\Tr\left(\left[ U^{\top}TU\right]^{-1}\begin{pmatrix}E^\top \\ 0 \end{pmatrix}
 \left[H^{\dagger} + \begin{pmatrix}E & 0 \end{pmatrix}
 \left[U^{\top}TU\right]^{-1}\begin{pmatrix}E^\top\\ 0\end{pmatrix}
 \right]^{-1}\begin{pmatrix}E & 0 \end{pmatrix} \left[U^{\top}TU\right]^{-1}
 \right)\\
 & = \Tr\left(T^{-1}\right)\\
&\quad -\Tr\left(U^{\top}T^{-1}U\begin{pmatrix}E^\top \\ 0 \end{pmatrix}
 \left[H^{\dagger} + \begin{pmatrix}E & 0 \end{pmatrix}U^{\top}
 T^{-1}U\begin{pmatrix}E^\top\\ 0\end{pmatrix}
 \right]^{-1}\begin{pmatrix}E & 0 \end{pmatrix}U^{\top} T^{-1}U
 \right)\\
 & \overset{(i)}= \Tr\left(T^{-1}\right)
 -\Tr\left(U^{\top}T^{-1}EE^\top
 \left[H^{\dagger} + EE^\top T^{-1}EE^\top
 \right]^{-1}EE^\top T^{-1}U
 \right)\\
 & = \Tr\left(T^{-1}\right)
 -\Tr\left(T^{-1}EE^\top
 \left[H^{\dagger} + EE^\top T^{-1}EE^\top
 \right]^{-1}EE^\top T^{-1}
 \right), \numberthis \label{e:trace_A_inverse_minus_trace_T}
 \end{align*}
 where $(i)$ follows since $(E;0)U^{\top} = (E;0)(E;F)^{\top}=EE^{\top}$. Now
\begin{align}\nonumber
0&\le \Tr\left(T^{-1}EE^\top
 \left[H^{\dagger} + EE^\top T^{-1}EE^\top
 \right]^{-1}EE^\top T^{-1}
 \right)\\
 &\le
\Tr\left(T^{-1}EE^\top
\left( EE^\top T^{-1}EE^\top \right)^{-1}
 EE^\top T^{-1} 
 \right) \nonumber \\
 &=
\Tr\left(T^{-1}EE^\top
\left( EE^\top \right)^{\dagger} T \left( EE^\top \right)^{\dagger}
 EE^\top T^{-1} 
 \right), 
 \label{e:H_inverse_psd}
 \end{align}
where the second inequality holds because
 \begin{align*}
 & H^{\dagger} \succeq 0 \\
 & \Rightarrow
  \left(EE^\top T^{-1}EE^\top
 \right)^{-1}
  -
 \left(H^{\dagger} + EE^\top T^{-1}EE^\top
 \right)^{-1} \succeq 0 \\
 & \Rightarrow
  T^{-1}EE^\top\left(\left(EE^\top T^{-1}EE^\top
 \right)^{-1}
  -
 \left(H^{\dagger} + EE^\top T^{-1}EE^\top
 \right)^{-1}\right)EE^\top T^{-1} \succeq 0 \\
&\Rightarrow
  T^{-1}EE^\top\left(EE^\top T^{-1}EE^\top
 \right)^{-1}EE^\top T^{-1} 
 \succeq
   T^{-1}EE^\top\left(
 \left(H^{\dagger} + EE^\top T^{-1}EE^\top
 \right)^{-1}\right)EE^\top T^{-1} 
 \end{align*}
along with the fact that, for
any symmetric positive semi-definite matrices $Q$ and $S$ such that
$Q \succeq S$, for all $i$,
$\mu_i(Q) \geq \mu_i(S)\ge 0$.
 
 Thus combining equations~\eqref{e:trace_A_inverse_minus_trace_T} and \eqref{e:H_inverse_psd} we get that
 \begin{align*}
     \left|\Tr(A^{-1})-\Tr\left(T^{-1}\right)\right| &\le \left|\Tr\left(T^{-1}EE^\top
\left( EE^\top \right)^{\dagger} T \left( EE^\top \right)^{\dagger}
 EE^\top T^{-1}
 \right)\right|.
 \end{align*}
 The rank of $T^{-1}EE^\top
\left( EE^\top \right)^{\dagger} T \left( EE^\top \right)^{\dagger}
 EE^\top T^{-1}$ is at most $k$, so
 \begin{align*}
    \left|\Tr(A^{-1})-\Tr\left(T^{-1}\right)\right| & \le  k \left\lv T^{-1}EE^\top
\left( EE^\top \right)^{\dagger} T \left( EE^\top \right)^{\dagger}
 EE^\top T^{-1}\right\rv_{op} \\
    & \le k \lv T^{-1}\rv_{op}^2\lv EE^\top (EE^{\top})^{\dagger}\rv_{op}^2\lv T\rv_{op} \\
    & \le \frac{k \mu_1(T)}{\mu_n(T)^2}. \numberthis \label{e:upper_bound_in_terms_of_k_star}
 \end{align*}
%Recall that $r_{k} \le R_{k}$ by Lemma~\ref{l:effective_ranks}. 
Next by invoking
Lemma~\ref{l:eigenvalue_bound}, for any $t>2n$, with probability at least $1-2e^{-t}$
 \begin{align*}
     \mu_1(T) & \le s_k\left[1+c\left(\frac{t+n}{r_{k}}+\sqrt{\frac{t+n}{R_{k}}}\right)\right] \\
     &\le s_k\left[1+3c\left(\frac{t}{r_{k}}+\sqrt{\frac{t}{R_{k}}}\right)\right] \\
     %& \le \sum_{j>k}\lambda_j\left[1+c\left(\frac{t}{r_{k}}+\frac{1}{c_0}+\sqrt{\frac{t}{R_{k}}+\frac{1}{c_0}}\right)\right] \\
     & \le c's_k\left[1+\left(\frac{t}{r_{k}}+\sqrt{\frac{t}{R_{k}}}\right)\right].
 \end{align*}
 Recall that $r_{k}\ge bn$ by the definition of the index $k$ in Definition~\ref{def:k_star}. Given a $\beta < c_0<1$, where $c_0$ is small enough, set $t = \min\left\{r_{k},R_{k}\right\}/\beta^2 = r_{k}/\beta^2$ (since $r_{k}\le R_{k}$ by Lemma~\ref{l:effective_ranks}) to get that
 \begin{align*}
     \mu_1(T) & \le c's_k\left[1+\left(\frac{1}{\beta^2}+\frac{1}{\beta}\right)\right]
 \end{align*}
 with probability at least $1-2\exp(- r_{k}/\beta^2)$. 
 Next, note that $p-k \ge \sum_{j>k}\lambda_{j}/\lambda_{k+1} = r_{k}\ge bn$. Therefore, by Lemma~\ref{l:smallest_singular_value}, for any $\beta<c_0<1$,
 \begin{align*}
    \Pr\left[ \mu_n(T) \le \beta s_k\right] 
      \le (c_1 \beta)^{c_2 \cdot r_{k}}.
 \end{align*}
  Combining the last two inequalities we find that, for any $\beta < c_0<1$
 \begin{align*}
     \frac{\mu_1(T)}{\mu_n(T)^2}& \le \frac{c'}{s_k}\left[\frac{1+\left(\frac{1}{\beta^2}+\frac{1}{\beta}\right)}{\beta^2} \right]
      \le \frac{c_3}{\beta^4s_k}
 \end{align*}
 with probability at least $1-2\exp(- r_{k}/\beta^2)-(c_1 \beta)^{c_2\cdot r_{k}}$. Combined with inequality~\eqref{e:upper_bound_in_terms_of_k_star} this completes our proof. 
\end{proof}
\subsubsection{$\E\left[\Tr(A^{-1})\right]$ is close to $\E\left[\Tr(T^{-1})\right]$}
Next, we show that $\E\left[\Tr(A^{-1})\right]$ is close to $\E\left[\Tr(T^{-1})\right]$.
\begin{lemma}\label{l:expectations_are_close}There exists a positive constant $c_0$ such that
\begin{align*}
    \left|\E\left[\Tr(A^{-1})\right]-\E\left[\Tr(T^{-1})\right]\right|& \le  \frac{c_0 k}{s_k}.
\end{align*}
\end{lemma}
\begin{proof}  Given any $\beta$ define 
\begin{align*}
    \omega = \frac{c k}{\beta^4 s_k} = \frac{c k}{\beta^4 r_k\lambda_{k+1}} .
\end{align*}
By Lemma~\ref{l:trace_A_close_to_trace_T} for any $\omega >\frac{c_1 k}{s_k}$, where $c_1$ is a large enough constant
\begin{align*}
    \Pr\left[ \left|\Tr(A^{-1})-\Tr(T^{-1})\right| >  \omega \right] \le 2\exp\left(-c'r_{k}^{3/2}\sqrt{\frac{ \lambda_{k+1}}{k}}\sqrt{\omega}\right) + \left(\frac{c'' k}{\omega s_k}\right)^{\frac{c_2\cdot  r_{k}}{4}}.
\end{align*}
Thus 
\begin{align*}
& \left| \E[\Tr(A^{-1}) ] - \E[\Tr(T^{-1})] ] \right| \\
&\qquad \leq \E\left[ \left| \Tr(A^{-1}) -\Tr(T^{-1}) \right| \right] \\
& \qquad  = \int_0^{\infty}
     \Pr\left[ \left| \Tr(A^{-1}) -\Tr(T^{-1}) \right| > \omega \right]\; \mathrm{d}\omega \\
     & \qquad  = \int_0^{\frac{c_1k}{s_k}}
     \Pr\left[ \left| \Tr(A^{-1}) -\Tr(T^{-1}) \right| > \omega \right]\; \mathrm{d}\omega +\int_{\frac{c_1k}{s_k}}^{\infty}
     \Pr\left[ \left| \Tr(A^{-1}) -\Tr(T^{-1}) \right| > \omega \right]\; \mathrm{d}\omega \\
& \qquad \le \frac{c_1k}{s_k}
  +\underbrace{\int_{{\frac{c_1k}{s_k}}}^{\infty}
  2 \exp\left(-c'r_{k}^{3/2}\sqrt{\frac{ \lambda_{k+1}}{k}}\sqrt{\omega}
     \right) \; \mathrm{d}\omega}_{=:\spadesuit}+\underbrace{\int_{{\frac{c_1k}{s_k}}}^{\infty}
\left(\frac{c'' k}{\omega s_k}\right)^{\frac{c_2 \cdot r_{k}}{4}}   \; \mathrm{d}\omega}_{=:\clubsuit}. \label{e:expected_value_decomposition} \numberthis
\end{align*}
First we control $\spadesuit$ as follows:
\begin{align*}
\spadesuit
& =   \int_{{\frac{c_1k}{s_k}}}^{\infty}
  2 \exp\left(-c'r_{k}^{3/2}\sqrt{\frac{ \lambda_{k+1}}{k}}\sqrt{\omega}
     \right) \; \mathrm{d}\omega \\
     & = 4 \exp\left(-c_3r_{k}\right) \frac{c_3 r_{k}+1}{\left(c'r_{k}^{3/2}\sqrt{\frac{\lambda_{k+1}}{k}}\right)^2} \hspace{0.3in}\mbox{(since $\int\exp(-\sqrt{z}) = -2e^{-\sqrt{z}}(\sqrt{z}+1) + c$)}\\
     & = \frac{4 k}{s_k} \left[\exp\left(-c_3r_{k}\right)\frac{c_3 r_{k}+1}{\left(c'r_{k}\right)^2}\right] \hspace{0.3in}\mbox{(since $s_k = r_k\lambda_{k+1}$)}\\&\le \frac{c_4 k}{s_k}.\numberthis \label{e:bound_spadesuit_bound_expectations}
\end{align*}
Next we control $\clubsuit$ as follows
\begin{align*}
    \clubsuit  = \int_{{\frac{c_1k}{s_k}}}^{\infty}
\left(\frac{c'' k}{\omega s_k}\right)^{\frac{c_2 \cdot r_{k}}{4}}   \; \mathrm{d}\omega 
& = \left(\frac{c'' k}{ s_k}\right)^{\frac{c_2 \cdot r_{k}}{4}} \int_{{\frac{c_1k}{s_k}}}^{\infty}
\left(\frac{1}{\omega }\right)^{\frac{c_2 \cdot r_{k}}{4}}   \; \mathrm{d}\omega \\
& = \left(\frac{c'' k}{ s_k}\right)^{\frac{c_2 \cdot r_{k}}{4}} \times \frac{1}{\frac{c_2 \cdot r_{k}}{4}-1}\times  
\left(\frac{s_k}{c_1k }\right)^{\frac{c_2 \cdot r_{k}}{4}-1}   \\
& = \frac{c_5 k}{s_k}\left((c'')^{\frac{c_2 \cdot r_{k}}{4}} \times \frac{4}{c_2 \cdot r_{k}-4}\times \left(\frac{1}{c_1}\right)^{\frac{c_2 \cdot r_{k}}{4}-1} \right)\\
& \le  \frac{c_5 k}{s_k} \numberthis \label{e:bound_clubsuit_bound_expectations}
\end{align*}
where the last inequality follows because the constant $c_1$ large enough and because $r_{k}\ge bn$ for a large enough constant $b$. Combining inequalities~\eqref{e:expected_value_decomposition}, \eqref{e:bound_spadesuit_bound_expectations} and \eqref{e:bound_clubsuit_bound_expectations} we conclude that
\begin{align*}
    \left| \E\left[\Tr(A^{-1}) \right] - \E\left[\Tr(T^{-1})\right]  \right|  & \le \frac{c_1k}{s_k}+\frac{c_4 k}{s_k}+ \frac{c_5 k}{s_k}  \le \frac{c_0 k}{s_k},
\end{align*}
wrapping up the proof.
\end{proof}
\subsubsection{$\Tr(T^{-1})$ concentrates around its mean}
Finally, we shall show that $\Tr(T^{-1})$ is close to its expectation $\E\left[\Tr(T^{-1})\right]$ with high probability.
\begin{lemma}\label{l:trace_T_close_to_its_expectation}
There exists a positive constant $c_0$ such that with probability at least $1-2e^{-n}$,
    \begin{align*}
    \left|\Tr(T^{-1})-\E\left[\Tr(T^{-1})\right]\right|& \le  \frac{c_0n}{s_k}\left[\sqrt{\frac{n}{R_{k}}}+\frac{n}{r_{k}}\right].   
    \end{align*}
\end{lemma}
\begin{proof}
% We will proceed by using 
We use a symmetrization argument:
\begin{align*}
    \left|\Tr(T^{-1}) - \E\left[\Tr(T^{-1})\right]\right| = \left|\sum_{i=1}^n \frac{1}{\mu_i(T)}-\E\left[\frac{1}{\mu_i(T)}\right]\right|
   &\le  \sum_{i=1}^n\left| \frac{1}{\mu_i(T)}-\E\left[\frac{1}{\mu_i(T)}\right]\right|\\
   &=  \sum_{i=1}^n\left| \frac{1}{\mu_i(T)}-\E_{T'}\left[\frac{1}{\mu_i(T')}\right]\right|,
\end{align*}
where in the equation above the matrices $T$ and $T'$ are independent and identically distributed. Thus
\begin{align*}
     \left|\Tr(T^{-1}) - \E\left[\Tr(T^{-1})\right]\right|  \le \sum_{i=1}^n\left| \E_{T'}\left[\frac{1}{\mu_i(T)}-\frac{1}{\mu_i(T')}\right]\right| 
     &\le \sum_{i=1}^n\E_{T'}\left[\left|\frac{1}{\mu_i(T)}-\frac{1}{\mu_i(T')}\right| \right] \\
     & = \sum_{i=1}^n\E_{T'}\left[\frac{\left|\mu_i(T')-\mu_i(T)\right|}{\mu_i(T)\mu_i(T')} \right]. \numberthis \label{e:trace_decomposition}
\end{align*}
% Now we know that 
By Lemma~\ref{l:eigenvalue_bound},
with probability at least $1-2e^{-n}$, for all $i\in [n]$,
\begin{align*}
s_k\left[1-c_1\left(\frac{n}{r_{k}}+\sqrt{\frac{n}{R_{k}}}\right)\right]\le \mu_i(T) \le  s_k\left[1+c_1\left(\frac{n}{r_{k}}+\sqrt{\frac{n}{R_{k}}}\right)\right]\label{e:C_near_p_iso}.\numberthis
\end{align*}
%where the last inclusion follows since $r_{k}\ge bn$ by the definition of the index $k$ in Definition~\ref{def:k_star} above and since $b$ is large enough.
We will assume that the event described above, which controls the singular values of $T$, occurs going forward. (This determines the success probability in the statement of the lemma.) 
The game plan now is to 
% assume that this event above occurs and 
evaluate the expectation with respect to $T'$ in equation~\eqref{e:trace_decomposition} by integrating tail bounds. 
% To execute this plan first note that by invoking Lemma~\ref{l:eigenvalue_bound} we find that 
Since \eqref{e:C_near_p_iso} holds,
\begin{align*}
    &\left|\mu_i(T)-\mu_i(T')\right| \\& = \max\{\mu_i(T)-\mu_i(T'),\mu_i(T')-\mu_i(T)\}\\
    & \le \max\left\{s_k\left[1+c_1\left(\frac{n}{r_{k}}+\sqrt{\frac{n}{R_{k}}}\right)\right]-\mu_i(T'),\right.\\ &\left.\qquad \qquad \qquad\mu_i(T')-s_k\left[1-c_1\left(\frac{n}{r_{k}}+\sqrt{\frac{n}{R_{k}}}\right)\right]\right\}\\
    & \le \max\left\{s_k\left[1+c_1\left(\frac{n}{r_{k}}+\sqrt{\frac{n}{R_{k}}}\right)\right]-s_k\left[1-c_2\left(\frac{t+n}{r_{k}}+\sqrt{\frac{t+n}{R_{k}}}\right)\right],\right.\\&\qquad\qquad \left.s_k\left[1+c_2\left(\frac{t+n}{r_{k}}+\sqrt{\frac{t+n}{R_{k}}}\right)\right]-s_k\left[1-c_1\left(\frac{n}{r_{k}}+\sqrt{\frac{n}{R_{k}}}\right)\right]\right\}\\
    & \hspace{2in} \mbox{(by Lemma~\ref{l:eigenvalue_bound})}\\
    &  \le c_3s_k \left(\frac{t+n}{r_{k}}+\sqrt{\frac{t+n}{R_{k}}}\right) \numberthis \label{e:eigenvalue_difference_high_prob_bound},
\end{align*}
with probability $1-2e^{-t}$.

Next, by Lemma~\ref{l:smallest_singular_value}, we know that for all $\beta <c_4<1$ 
\begin{align}
    \Pr\left[\mu_n(T')
    \le \beta s_k\right] \le (c_5 \beta)^{c_6\cdot r_{k}}. \label{e:T'_smallest_eigenvalue_bound}
\end{align}
Combining equations~\eqref{e:eigenvalue_difference_high_prob_bound} and \eqref{e:T'_smallest_eigenvalue_bound}, and because condition~\eqref{e:C_near_p_iso} holds, we get that 
\begin{align*}
    \Pr\left[\exists\; i\in [n]\;:\;\frac{|\mu_i(T)-\mu_i(T')|}{\mu_i(T)\mu_i(T')}\ge \frac{c_3 \left(\frac{t+n}{r_{k}}+\sqrt{\frac{t+n}{R_{k}}}\right)}{\beta  s_k \left[1-c_1\left(\frac{n}{r_{k}}+\sqrt{\frac{ n}{R_{k}}}\right)\right]}\right] \le 2e^{-t} + (c_5 \beta)^{c_6\cdot r_{k}}.
\end{align*}
Now since $r_{k} \ge bn$ for a large enough constant $b$ by the definition of $k$, and since $R_{k}>r_{k}$ by Lemma~\ref{l:effective_ranks}, we can simplify the denominator in the equation above to get that
\begin{align*}
     \Pr\left[\exists\; i\in [n]\;:\;\frac{|\mu_i(T)-\mu_i(T')|}{\mu_i(T)\mu_i(T')}\ge \frac{c_7 \left(\frac{t+n}{r_{k}}+\sqrt{\frac{t+n}{R_{k}}}\right)}{\beta s_k }\right] \le 2e^{-t} + (c_5 \beta)^{c_6\cdot r_{k}}.
\end{align*}
% Set $t = n/\beta$ to get that
Setting $t = n/\beta$ yields
\begin{align*}
    &\Pr\left[\exists\; i\in [n]\;:\;\frac{|\mu_i(T)-\mu_i(T')|}{\mu_i(T)\mu_i(T')}\ge \frac{c_7 \left(\frac{n(\beta+1)}{\beta r_{k}}+\sqrt{\frac{n(\beta+1)}{\beta R_{k}}}\right)}{\beta s_k }\right] \le 2e^{-n/\beta} + (c_5 \beta)^{c_6\cdot r_{k}}.
    \end{align*}
    Now since $\beta <c_4<1$, we find that
\begin{align*}
    &\Pr\left[\exists\; i\in [n]\;:\;\frac{|\mu_i(T)-\mu_i(T')|}{\mu_i(T)\mu_i(T')}\ge \frac{c_8}{ s_k } \max\left\{\frac{n}{\beta^2 r_{k}},\frac{\sqrt{n}}{\beta^{3/2} \sqrt{R_{k}}}\right\}\right]\\& \hspace{3.5in}\le 2e^{-n/\beta} + (c_5 \beta)^{c_6\cdot r_{k}}.\label{e:tail_bound_the_difference} \numberthis
    \end{align*}
    %Furthermore, since $r_{k} \geq b n$
    % for a large enough constant $b$, we get
    % \begin{align*}
    % &\Pr\left[\exists\; i\in [n]\;:\;\frac{|\mu_i(T)-\mu_i(T')|}{\mu_i(T)\mu_i(T')}\ge \frac{c_7' \left(\frac{n}{\beta r_{k}}+\sqrt{\frac{n}{\beta r_{k}}}\right)}{\beta s_k }\right] \le 2e^{-n/\beta} + (c_5 \beta)^{c_6\cdot r_{k}}.
    % \end{align*}
%     \begin{align*}
%     \Pr\left[\exists\; i\in [n]\;:\;\frac{|\mu_i(T)-\mu_i(T')|}{\mu_i(T)\mu_i(T')}\ge \frac{c_8 \sqrt{n}}{\beta^2 \sqrt{r_{k}}s_k}\right] \le 2e^{-n/\beta} + (c_5 \beta)^{c_6\cdot r_{k}} . \label{e:tail_bound_the_difference} \numberthis
% \end{align*}
 For every $\beta$ define
\begin{align*}
    \omega := \frac{c_8}{ s_k } \max\left\{\frac{n}{\beta^2 r_{k}},\frac{\sqrt{n}}{\beta^{3/2} \sqrt{R_{k}}}\right\}.
\end{align*}
Inverting the map from $\beta$ to $\omega$ yields
\begin{align*}
    \beta(\omega) = \begin{cases}\left(\frac{c_8 \sqrt{n}}{\omega \sqrt{R_{k}}s_k}\right)^{2/3} & \text{if } \omega \leq \omega_{\tau} :=\frac{c_8}{s_k}\left(\frac{r_{k}^3}{R_{k}^2 n}\right), \\
    \sqrt{\frac{c_8 n}{\omega r_k s_k }}& \text{otherwise.}
    \end{cases} \numberthis \label{e:beta_formula}
\end{align*} 
Let $\omega_0$ be such that $\beta(\omega_0) = c_4$, and define
\begin{align*}
    \omega_{-} := \min\left\{\omega_0,\omega_{\tau}\right\} \quad \text{and} \quad 
    \omega_{+} := \max\left\{\omega_0,\omega_{\tau}\right\}.
\end{align*}
Applying inequality~\eqref{e:tail_bound_the_difference} we have that, for all $\omega \in \left( \omega_{-}, \omega_{\tau}\right]$
\begin{align*}
    &\Pr\left[\exists \; i\in [n]\;:\;\frac{|\mu_i(T)-\mu_i(T')|}{\mu_i(T)\mu_i(T')}\ge \omega\right] \\&\qquad \le 2\exp\left(-c_9\left(\omega n \sqrt{R_{k}}s_k\right)^{2/3}\right) + \left( \frac{c_{10} \sqrt{n}}{\omega \sqrt{R_{k}}s_k}\right)^{c_{11}\cdot r_{k}}, \numberthis \label{e:mu_concentration_bound_first_part}
\end{align*}
and for 
$\omega > \omega_{+}$, we have
\begin{align*}
 \Pr\left[\exists \; i\in [n]\;:\;\frac{|\mu_i(T)-\mu_i(T')|}{\mu_i(T)\mu_i(T')}\ge \omega\right] &
 \le  
 2\exp\left(-c_{12}\left(\omega n r_k  s_k \right)^{1/2}\right) 
 + \left( \frac{c_{13} n}{\omega r_{k}s_k}\right)^{c_{14}\cdot r_{k}}. \numberthis \label{e:mu_concentration_bound_second_part}
\end{align*}

Thus
\begin{align*}
    \E_{T'}\left[\frac{\left|\mu_i(T')-\mu_i(T)\right|}{\mu_i(T)\mu_i(T')} \right] 
    & = \int_{0}^{\infty}\Pr\left[\frac{\left|\mu_i(T')-\mu_i(T)\right|}{\mu_i(T)\mu_i(T')}  \ge \omega\right]  \; \mathrm{d}\omega \\
    & = \int_{0}^{\omega_0} \Pr\left[\frac{\left|\mu_i(T')-\mu_i(T)\right|}{\mu_i(T)\mu_i(T')}  \ge \omega\right]  \; \mathrm{d}\omega \\
    & \hspace{0.5in}
    +\int_{\omega_-}^{\omega_{\tau}}\Pr\left[\frac{\left|\mu_i(T')-\mu_i(T)\right|}{\mu_i(T)\mu_i(T')}  \ge \omega\right]  \; \mathrm{d}\omega\\
    & \hspace{0.5in} +\int_{\omega_+}^{\infty}\Pr\left[\frac{\left|\mu_i(T')-\mu_i(T)\right|}{\mu_i(T)\mu_i(T')}  \ge \omega\right]  \; \mathrm{d}\omega\\
        & \le \omega_0
    +\underbrace{\int_{\omega_-}^{\omega_{\tau}}\Pr\left[\frac{\left|\mu_i(T')-\mu_i(T)\right|}{\mu_i(T)\mu_i(T')}  \ge \omega\right]  \; \mathrm{d}\omega}_{=:\spadesuit}\\
    & \hspace{0.5in} +\underbrace{\int_{\omega_+}^{\infty}\Pr\left[\frac{\left|\mu_i(T')-\mu_i(T)\right|}{\mu_i(T)\mu_i(T')}  \ge \omega\right]  \; \mathrm{d}\omega}_{=:\clubsuit}\numberthis. \label{e:expectation_difference_decomposition}
\end{align*}
Let us perform each of these two integrals $\spadesuit$ and $\clubsuit$ separately. 

First, 
\begin{align*}
    &\spadesuit \\ & =\int_{\omega_-}^{\omega_{\tau}}\Pr\left[\frac{\left|\mu_i(T')-\mu_i(T)\right|}{\mu_i(T)\mu_i(T')}  \ge \omega\right]  \; \mathrm{d}\omega \\
    & \le \int_{\omega_-}^{\omega_{\tau}}\left[2\exp\left(-c_9\left(\omega n  \sqrt{R_{k}}s_k\right)^{2/3}\right) + \left( \frac{c_{10} \sqrt{n}}{\omega \sqrt{R_{k}}s_k}\right)^{c_{11}\cdot r_{k}} \right]  \; \mathrm{d}\omega \hspace{0.3 in} \mbox{(by inequality~\eqref{e:mu_concentration_bound_first_part})}\\
    & \le \mathbb{I}[\omega_{-}<\omega_{\tau}]\int_{\omega_-}^{\infty}\left[2\exp\left(-c_9\left(\omega n  \sqrt{R_{k}} s_k\right)^{2/3}\right) + \left( \frac{c_{10} \sqrt{n}}{\omega \sqrt{R_{k}}s_k}\right)^{c_{11}\cdot r_{k}} \right]  \; \mathrm{d}\omega .
\end{align*}
Now, for $\zeta := c_9\left(n  \sqrt{R_{k}}  s_k\right)^{2/3}$, 
we have that 
\begin{align*}
  &  2\int_{\omega_-}^{\infty}\exp\left(-c_9\left(\omega n  \sqrt{R_{k}} s_k\right)^{2/3}\right)   \; \mathrm{d}\omega \\
      &  =2\int_{\omega_-}^{\infty}\exp\left(-\zeta \omega^{2/3} \right)   \; \mathrm{d}\omega \\
    & = 
      \frac{3\omega_{-}^{1/3}\exp(-\zeta\omega_{-}^{2/3})}{ \zeta}
      +\frac{3\sqrt{\pi}\left(1 - \mathsf{erf}\left(\sqrt{\zeta}\omega_{-}^{1/3}\right)\right)}{2\zeta^{3/2}} \\ & \hspace{2in}\mbox{(since $\int\exp(-z^{2/3}) = \frac{3}{4}\left(\sqrt{\pi}\mathsf{erf}(z^{1/3})-2e^{-z^{2/3}}z^{1/3}\right) + c$)} \\
    % \frac{3\omega_{-}^{1/3}\exp(-\zeta\omega_{-}^{2/3})}{\zeta}-\frac{3\sqrt{\pi}\mathsf{erf}\left(\sqrt{\zeta}\omega_{-}^{1/3}\right)}{2a^{2/3}}\\ 
    % https://www.comm.utoronto.ca/frank/notes/erfc.pdf
    &  \le \frac{3\omega_{-}^{1/3}\exp(-\zeta\omega_{-}^{2/3})}{\zeta}
          + \frac{3 \exp(-\zeta \omega_-^{2/3})}{2\zeta^2 \omega_{-}^{1/3}}
       \\
    &   =  
        \frac{c_{15}\exp\left(-c_9\left(n  \sqrt{R_{k}}  s_k\right)^{2/3}\omega_{-}^{2/3}\right)}{(n\sqrt{R_k}s_k)^{2/3}}\left( 
    \omega_{-}^{1/3}
          + \frac{1 }{\left(n  \sqrt{R_{k}}  s_k\right)^{2/3}\omega_{-}^{1/3}}
        \right).
    \numberthis
    \label{e:first_term_of_spade}
\end{align*}
%since $R_{k} \geq r_{k}$.
Continuing our work of bounding $\spadesuit$, we have that
\begin{align*}
 \int_{\omega_{-}}^{\infty}
\left( \frac{c_{10} \sqrt{n}}{\omega \sqrt{R_{k}}s_k}\right)^{c_{11}\cdot r_{k}}  \; \mathrm{d}\omega
& =
\left( \frac{c_{10} \sqrt{n}}{ \sqrt{R_{k}}s_k}\right)^{c_{11}\cdot r_{k}}
\int_{\omega_{-}}^{\infty}
\left( \frac{1}{\omega}\right)^{c_{11}\cdot r_{k}}  \; \mathrm{d}\omega   \\
& =
\left( \frac{c_{10} \sqrt{n}}{ \sqrt{R_{k}}s_k}\right)^{c_{11}\cdot r_{k}}
\times
\frac{1}{c_{11}\cdot r_{k}-1}
\left( 
\frac{1}{\omega_{-}}
\right)^{c_{11}\cdot r_{k}-1} \\
& \le c_{16} \left( \frac{c_{10} \sqrt{n}}{ \sqrt{R_{k}}s_k}\right)^{c_{11}\cdot r_{k}}
\left( 
\frac{1}{\omega_{-}}
\right)^{c_{11}\cdot r_{k}-1}, \numberthis
    \label{e:second_term_of_spade}
\end{align*}
where the last inequality follows since $r_{k}\ge bn$ for a large enough constant $b$. By combining inequalities \eqref{e:first_term_of_spade} and \eqref{e:second_term_of_spade} we get the following bound on the integral $\spadesuit$:
\begin{align*}
    \spadesuit & \le \mathbb{I}[\omega_{-}<\omega_{\tau}]\frac{c_{15}\exp\left(-c_9\left(n  \sqrt{R_{k}}  s_k\right)^{2/3}\omega_{-}^{2/3}\right)}{(n\sqrt{R_k}s_k)^{2/3}}\left( 
    \omega_{-}^{1/3}
          + \frac{1 }{\left(n  \sqrt{R_{k}}  s_k\right)^{2/3}\omega_{-}^{1/3}}
        \right)\\&\hspace{1.5in}+\mathbb{I}[\omega_{-}<\omega_{\tau}]c_{16} \left( \frac{c_{10} \sqrt{n}}{ \sqrt{R_{k}}s_k}\right)^{c_{11}\cdot r_{k}}
\left( 
\frac{1}{\omega_{-}}
\right)^{c_{11}\cdot r_{k}-1}. \label{e:concentration_of_T_spade_suit_bound} \numberthis
\end{align*}

Let us now bound $\clubsuit$
\begin{align*}
    \clubsuit& = \int_{\omega_+}^{\infty}\Pr\left[\frac{\left|\mu_i(T')-\mu_i(T)\right|}{\mu_i(T)\mu_i(T')}  \ge \omega\right]  \; \mathrm{d}\omega \\
    & \le \int_{\omega_+}^{\infty}\left[ 2\exp\left(-c_{12}\left(\omega n r_k  s_k \right)^{1/2}\right) 
 + \left( \frac{c_{13} n}{\omega r_{k}s_k}\right)^{c_{14}\cdot r_{k}}\right]  \; \mathrm{d}\omega \hspace{0.3in}\mbox{(applying inequality~\eqref{e:mu_concentration_bound_second_part}).}
\end{align*}
For $\zeta' := c_{12}\left(n r_{k} s_k\right)^{1/2}$, we have
\begin{align*}
   &2 \int_{\omega_+}^{\infty}\exp\left(-c_{12}\left(\omega n r_k  s_k \right)^{1/2}\right)  \; \mathrm{d}\omega  \\
   &  \qquad = 2 \int_{\omega_+}^{\infty}\exp\left(-\zeta'\omega^{1/2} \right)  \; \mathrm{d}\omega \\
   &  \qquad  = \frac{4\exp(-\zeta'\sqrt{\omega_{+}})(\zeta'\sqrt{\omega_{+}}+1)}{\zeta'^2} \hspace{0.3in}\mbox{(since $\int\exp(-\sqrt{z}) = -2e^{-\sqrt{z}}(\sqrt{z}+1) + c$)}\\
   &  \qquad  = \frac{c_{17}\exp\left(-c_{12}\left(nr_{k}s_k\omega_{+}\right)^{1/2}\right)\left[c_{12}\left(nr_{k}s_k\omega_{+}\right)^{1/2}+1\right]}{n r_{k}s_k} .  \numberthis
    \label{e:first_term_of_club}
\end{align*}
We continue to bound the other integral in $\clubsuit$ as follows
\begin{align*}
    \int_{\omega_+}^{\infty}\left( \frac{c_{13} n}{\omega r_{k}s_k}\right)^{c_{14}\cdot r_{k}}  \; \mathrm{d}\omega & \le c_{18} \left( \frac{c_{13} n}{ r_{k}s_k}\right)^{c_{14}\cdot r_{k}}
\left( 
\frac{1}{\omega_{+}}
\right)^{c_{14}\cdot r_{k}-1}, \numberthis
    \label{e:second_term_of_club}
\end{align*}
where the bound follows by mirroring the logic used to arrive at inequality~\eqref{e:second_term_of_spade} above. Therefore, combining inequalities~\eqref{e:first_term_of_club} and \eqref{e:second_term_of_club} we get that
\begin{align*}
        \clubsuit &\le \frac{c_{17}\exp\left(-c_{12}\left(nr_{k}s_k\omega_{+}\right)^{1/2}\right)\left[c_{12}\left(nr_{k}s_k\omega_{+}\right)^{1/2}+1\right]}{n r_{k}s_k}\\ &\hspace{2in}+c_{18} \left( \frac{c_{13} n}{ r_{k}s_k}\right)^{c_{14}\cdot r_{k}}
        \left( 
\frac{1}{\omega_{+}}
\right)^{c_{14}\cdot r_{k}-1}. \label{e:concentration_of_T_club_suit_bound} \numberthis
\end{align*}

Having controlled both $\spadesuit$ and $\clubsuit$ in \eqref{e:concentration_of_T_spade_suit_bound} and \eqref{e:concentration_of_T_club_suit_bound} respectively, by using the decomposition in \eqref{e:expectation_difference_decomposition} we find that
\begin{align*}
    &\E_{T'}\left[\frac{\left|\mu_i(T')-\mu_i(T)\right|}{\mu_i(T)\mu_i(T')} \right] \\& \qquad \le \omega_0 
     + \mathbb{I}[\omega_{-}<\omega_{\tau}]\frac{c_{15}\exp\left(-c_9\left(n  \sqrt{R_{k}}  s_k\right)^{2/3}\omega_{-}^{2/3}\right)}{(n\sqrt{R_k}s_k)^{2/3}}\left( 
    \omega_{-}^{1/3}
          + \frac{1 }{\left(n  \sqrt{R_{k}}  s_k\right)^{2/3}\omega_{-}^{1/3}}
        \right)\\&\hspace{0.8in}+\mathbb{I}[\omega_{-}<\omega_{\tau}]c_{16} \left( \frac{c_{10} \sqrt{n}}{ \sqrt{R_{k}}s_k}\right)^{c_{11}\cdot r_{k}}
\left( 
\frac{1}{\omega_{-}}
\right)^{c_{11}\cdot r_{k}-1} \\
&\hspace{0.8in}+\frac{c_{17}\exp\left(-c_{12}\left(nr_{k}s_k\omega_{+}\right)^{1/2}\right)\left[c_{12}\left(nr_{k}s_k\omega_{+}\right)^{1/2}+1\right]}{n r_{k}s_k}\\&\hspace{0.8in}+c_{18} \left( \frac{c_{13} n}{ r_{k}s_k}\right)^{c_{14}\cdot r_{k}}\left( 
\frac{1}{\omega_{+}}
\right)^{c_{14}\cdot r_{k}-1}. \numberthis \label{e:expectation_final_inequality}
\end{align*}

We now consider two cases. 

\textbf{Case 1:} $\left(\omega_0 < \omega_{\tau}\right)$. In this case, using the fact that $\beta(\omega_0) = c_4$ and the formula for $\beta$ in equation~\eqref{e:beta_formula} we get that
\begin{align*}
% \label{e:w0_value}
    \omega_0 = \frac{c_8 \sqrt{n}}{c_4^{3/2}\sqrt{R_{k}}s_k} = \frac{c_{19}\sqrt{n}}{\sqrt{R_{k}}s_k},
\end{align*}
and that
\begin{align*}
    \omega_- &= \min\{\omega_0,\omega_{\tau}\} = \omega_0=\frac{c_8 \sqrt{n}}{c_4^{3/2}\sqrt{R_{k}}s_k}, \\
    \omega_+ & = \max\{\omega_0,\omega_{\tau}\} = \omega_{\tau} = \frac{c_8 r_{k}^3}{R_{k}^2 ns_k}.
\end{align*}
Also note that in this case since,
\begin{align*}
    &\omega_0 = \frac{c_8 \sqrt{n}}{c_4^{3/2}\sqrt{R_{k}}s_k} <  \frac{c_8 r_{k}^3}{R_{k}^2 ns_k} = \omega_{\tau}   \quad \Rightarrow R_{k} \le \frac{c_4 r_{k}^{2}}{n}
\end{align*}
and so $\omega_{+} \ge \frac{c_8 n }{c_4^2 r_{k}s_k}.$

Thus, substituting the above values
of $\omega_0$, $\omega_-$ in inequality~\eqref{e:expectation_final_inequality}, 
and, because the RHS of this inequality is a decreasing function
in $\omega_{+}$
(since the function $z \mapsto \exp(-z)(z+1)$ is a decreasing function for all positive $z$),
replacing $\omega_{+}$ with the above lower bound,
we find that 
\begin{align*}
   &\E_{T'}\left[\frac{\left|\mu_i(T')-\mu_i(T)\right|}{\mu_i(T)\mu_i(T')} \right] \\
    &\qquad  \le 
     \frac{c_{19}\sqrt{n}}{\sqrt{R_{k}}s_k} 
     + \frac{c_{20}\exp(-c_{21}n)}{\sqrt{n R_{k}}s_k}
     +\frac{c_{20}\exp(-c_{21}n)}{n^{3/2}\sqrt{R_{k}}s_k} \\ 
     &\qquad\qquad \qquad +\frac{c_{16} c_8 \sqrt{n}}{c_4^{3/2} \sqrt{R_{k}}s_k}
                      \left(\frac{c_{10}c_4^{3/2}}{c_8}\right)^{c_{11}\cdot r_{k}} 
     + \frac{c_{22}\exp(-c_{23}n)}{r_{k}s_k} + \frac{c_{18} c_8 n }{c_4^2 r_{k}s_k} \left(\frac{c_{13}c_{4}^2}{c_8}\right)^{c_{14}\cdot r_{k}} \\
   &\qquad  \overset{(i)}{\le}
     \frac{c_{19}\sqrt{n}}{\sqrt{R_{k}}s_k} 
     + \frac{c_{20}\exp(-c_{21}n)}{\sqrt{n R_{k}}s_k}
     +\frac{c_{20}\exp(-c_{21}n)}{n^{3/2}\sqrt{R_{k}}s_k}  +\frac{c_{16} c_8 \sqrt{n}}{c_4^{3/2} \sqrt{R_{k}}s_k}
     + \frac{c_{22}\exp(-c_{23}n)}{r_{k}s_k} + \frac{c_{18} c_8 n }{c_4^2 r_{k}s_k}  \\
    & \qquad \le\frac{c_{24}\sqrt{n}}{\sqrt{R_{k}}s_k} 
+ \frac{c_{25}n}{r_{k}s_k},
\end{align*}
where $(i)$ follows since $c_4$ is small enough. This combined with inequalities~\eqref{e:trace_decomposition} and \eqref{e:C_near_p_iso} proves the lemma in this case.

\textbf{Case 2:} $\left(\omega_0 \ge \omega_{\tau}\right)$. In this case, using the fact that $\beta(\omega_0) = c_4$ and the formula for $\beta$ in equation~\eqref{e:beta_formula} we get that
\begin{align*}
 %   \omega_0 = \frac{c_{8} n}{c_4^{1/2}r_{k}s_k} 
    \omega_0 = \frac{c_{8} n}{c_4^2 r_{k}s_k} 
\end{align*}
and that
\begin{align*}
    \omega_- &= \min\{\omega_0,\omega_{\tau}\} = \omega_{\tau}, \\
    \omega_+ & = \max\{\omega_0,\omega_{\tau}\} = \omega_{0} = \frac{c_{8}n}{c_4^2 r_{k}s_k}.
\end{align*}

Now by applying inequality~\eqref{e:expectation_final_inequality} we get that
\begin{align*}
    \E_{T'}\left[\frac{\left|\mu_i(T')-\mu_i(T)\right|}{\mu_i(T)\mu_i(T')} \right] 
    &\le \frac{c_{8}n}{c_4^2 r_{k}s_k}+ \frac{c_{26}\exp(-c_{27}n)}{r_{k}s_k}+\frac{c_{28}n}{r_{k}s_k}\left(\frac{c_{13}c_4^2}{c_{8}}\right)^{c_{14}\cdot r_{k}} \\
    &\overset{(i)}{\le} \frac{c_{8}n}{c_4^2 r_{k}s_k}+ \frac{c_{26}\exp(-c_{27}n)}{r_{k}s_k}+\frac{c_{28}n}{r_{k}s_k} \le \frac{c_{29} n}{r_{k}s_k},
\end{align*}
where $(i)$ follows since $c_4$ is small enough. Again, combining this inequality with inequalities~\eqref{e:trace_decomposition} and \eqref{e:C_near_p_iso} proves the lemma in this second case.
\end{proof}

\subsubsection{Bounds on $\E\left[\Tr(A^{-1})\right]$}
To characterize $\Tr\left(A^{-1}\right)$ in terms of relevant problem parameters we will  need to establish upper and lower bounds that are tight up to the leading constant on its expectation. 
\begin{lemma}
\label{l:upper_lower_bounds_on_expectation_of_Tr_A_inverse}There are positive constants $c_0$ and $c_1$ such that
\begin{align*}
 \left|\E\left[\Tr\left(A^{-1}\right)\right] - \frac{n}{s_k} \right|&\le \frac{c_0 n}{s_k}  \left[  \sqrt{\frac{n}{R_{k}}}+\frac{n}{r_{k}}+\frac{ k}{n}+e^{-c_1\sqrt{n}} \right] .
\end{align*}
\end{lemma}
\begin{proof}By Lemma~\ref{l:expectations_are_close} we know that 
\begin{align}\label{e:expected_value_of_A_sandwich_inequality}
     \E\left[\Tr(T^{-1})\right]-\frac{c k}{s_k}\le \E\left[\Tr(A^{-1})\right] \le \E\left[\Tr(T^{-1})\right]+\frac{c k}{s_k}.
\end{align}
Thus, we shall instead upper and lower bound $\E\left[\Tr(T^{-1})\right]$. 

\paragraph{The lower bound:} By definition
\begin{align*}
    \E\left[\Tr(T^{-1})\right] & = \E\left[\sum_{i=1}^n\frac{1}{\mu_i(T)}\right]
     \ge \E\left[\frac{n}{\frac{1}{n}\sum_{i=1}^n\mu_i(T)}\right] \hspace{0.3in}\mbox{(by the AM-HM inequality).} \numberthis \label{e:expectation_lower_bound_AM_HM}
\end{align*}
By Bernstein's inequality (see Theorem~\ref{thm:bernstein}) we know that with probability at least $1-2e^{-t}$,
\begin{align*}
    \frac{1}{n}\sum_{i=1}^n\mu_i(T)  &= \frac{1}{n}\Tr(T)\\
   &  = \frac{1}{n}\sum_{i>k} \lambda_i \Tr(z_iz_i^{\top})\\
    & = \frac{1}{n}\sum_{i>k} \lambda_i \lv z_i\rv^2\\
    & \le \sum_{i>k} \lambda_i + c_2 \max\left\{t\lambda_{k+1},\sqrt{t\sum_{i>k} \lambda_i^2}\right\}\\
    & = s_k\left[1 + c_2 \max\left\{\frac{t}{r_{k}},\sqrt{\frac{t}{R_{k}}}\right\}\right] \hspace{0.3in} \mbox{(since $s_{k}=\sum_{j>k}\lambda_j $)}\\
    % & \le s_k \lambda_i\left[1 + c_2 \max\left\{\frac{t}{\sqrt{R_{k}}},\sqrt{\frac{t}{R_{k}}}\right\}\right] \hspace{0.3in} \mbox{(since $r_{k}\ge \sqrt{R_{k}}$ by Lemma~\ref{l:effective_ranks})}\\
    % &=s_k\left[1 + c_2 \max\left\{\frac{t}{\sqrt{R_{k}}},\sqrt{\frac{t}{R_{k}}}\right\}\right]. %\hspace{0.3in} \mbox{(since $r_{k}\ge \sqrt{R_{k}}$ by Lemma~\ref{l:effective_ranks})}.
        & \le s_k \left[1 + c_2 \max\left\{\frac{t}{\sqrt{R_{k}}},\sqrt{\frac{t}{R_{k}}}\right\}\right],
        \end{align*}
since $r_{k}\ge \sqrt{R_{k}}$ by Lemma~\ref{l:effective_ranks}.
Setting $t = \sqrt{n}$ implies that
\begin{align*}
    \frac{1}{n}\sum_{i=1}^n\mu_i(T) & \le s_k\left[1 + 2c_2\sqrt{\frac{n}{R_{k}}}\right]
\end{align*}
with probability at least $1-2e^{-\sqrt{n}}$. Thus by inequality~\eqref{e:expectation_lower_bound_AM_HM}
   \begin{align} \label{e:lower_bound_expected_trace_T}
  \nonumber  \E\left[\Tr(T^{-1})\right] & \ge \frac{n }{s_k\left(1+2c_2\sqrt{\frac{n}{R_{k}}}\right)} \Pr\left[ \frac{1}{n}\sum_{i=1}^n\mu_i(T) \le s_k\left[1 + 2c_2\sqrt{\frac{n}{R_{k}}}\right] \right]\nonumber\\
&\ge     \frac{n}{s_k}\left[\frac{ 1-2e^{-\sqrt{n}}}{1+2c_2\sqrt{\frac{n}{R_{k}}}}\right]. 
\end{align}
Combined with the lower bound in inequality~\eqref{e:expected_value_of_A_sandwich_inequality} we find that
\begin{align*}
    \E\left[\Tr(A^{-1})\right]& \ge \frac{n}{s_k}\left[\frac{ 1-2e^{-\sqrt{n}}}{1+2c_2\sqrt{\frac{n}{R_{k}}}}\right]-\frac{c_{1} k}{s_k} \\
    & \ge \frac{n}{s_k}\left[1-\frac{ 2c_2\sqrt{\frac{n}{R_{k}}}+2e^{-\sqrt{n}}}{1+2c_2\sqrt{\frac{n}{R_{k}}}}-\frac{c_{1} k}{n}\right] \\
    & \ge \frac{n}{s_k} \left[1-c_0\left( \sqrt{\frac{n}{R_{k}}}+\frac{ k}{n}+e^{-\sqrt{n}}\right)\right] \hspace{0.3in} \mbox{(since $R_{k} \ge r_{k}\ge bn$)}. \label{e:expectation_trace_A_lower_bound}\numberthis
\end{align*}
This proves the desired lower bound.

\paragraph{The upper bound:} To obtain the upper bound we shall bound
\begin{align}
\E\left[\Tr(T^{-1})\right] & = \E\left[\sum_{i=1}^n\frac{1}{\mu_i(T)}\right]  \le n \E\left[\frac{1}{\mu_n(T)}\right].\label{e:upper_bound_expected_trace}
\end{align}
We will upper bound the expected value of $1/\mu_n(T)$ again by integrating tail bounds. 
% So that,
We have
\begin{align*}
    \E\left[\frac{1}{\mu_n(T)}\right]& = \int_{0}^\infty \Pr\left[\frac{1}{\mu_n(T)}\ge \omega\right] \; \mathrm{d} \omega\\
    & = \int_{0}^\infty \Pr\left[\mu_n(T)\le \frac{1}{\omega}\right] \; \mathrm{d} \omega\\
    & = \underbrace{\int_{0}^{\frac{1}{ s_k  \left[1-c_3 \left(\frac{n+\eta}{r_{k}}+\sqrt{\frac{n+\eta}{R_{k}}}\right)\right]}} \Pr\left[\mu_n(T)\le \frac{1}{\omega}\right] \; \mathrm{d} \omega}_{=:\clubsuit}\\&\qquad+\underbrace{\int_{\frac{1}{s_k  \left[1-c_3 \left(\frac{n+\eta}{r_{k}}+\sqrt{\frac{n+\eta}{R_{k}}}\right)\right]}}^{\frac{1}{c_4 s_k}}\Pr\left[\mu_n(T)\le \frac{1}{\omega}\right] \; \mathrm{d} \omega}_{=:\spadesuit}\\&\qquad +\underbrace{\int_{\frac{1}{c_4 s_k}}^{\infty} \Pr\left[\mu_n(T)\le \frac{1}{\omega}\right] \; \mathrm{d} \omega}_{=: \vardiamond}, \numberthis \label{e:decomposition_into_suits}
\end{align*}
where 
\begin{itemize}
    \item $c_3$ is the constant $c$ from Lemma~\ref{l:eigenvalue_bound},
    \item $c_4$ is smaller than the constant $c_1^2$ in Lemma~\ref{l:smallest_singular_value}, and
    \item $\eta$ is small enough such that it satisfies $c_4 \le 1-c_3\left(\frac{n+\eta}{r_{k}}+\sqrt{\frac{n+\eta}{R_{k}}}\right)$. 
\end{itemize}
Below we will set $\eta$ to scale linearly with $n$, thus, this condition will be satisfied since $R_k\ge r_{k} \ge bn$ for a large enough value of $b$.

The first term $\clubsuit$ is positive because 
$\eta$ scales linearly with $n$ and $R_k\ge r_k\ge bn$ for suitably large $b$, 
and so it can be bounded as follows:
\begin{align} 
    \clubsuit & \le \frac{1}{ s_k \left[1-c_3 \left(\frac{n+\eta}{r_{k}}+\sqrt{\frac{n+\eta}{R_{k}}}\right)\right]}.\label{e:clubsuit_expected_value_upper_bound}
\end{align}
%where the last inequality follows since we will set $\eta$ to scale linearly with $n$ and because $r_{k}\ge bn$, where $b$ is large enough.

Next, consider the term $\spadesuit$. Here we will use the additive concentration inequality (Lemma~\ref{l:eigenvalue_bound}). By Lemma~\ref{l:eigenvalue_bound} we know that  with probability at most $2e^{-t}$ 
\begin{align*}
    \mu_n(T) &\le                                                       s_k \left[1-c_3\left(\frac{t+n}{r_{k}}+\sqrt{\frac{t+n}{R_{k}}}\right)\right] \\
    &\le s_k \left[1-c_3\left(\frac{t+n}{r_{k}}+\sqrt{\frac{t+n}{r_{k}}}\right)\right] \hspace{0.3in}\mbox{(since $r_{k}\le R_{k}$ by Lemma~\ref{l:effective_ranks})}\\
     &\le s_k \left[1-2c_3\max\left\{\frac{t+n}{r_{k}},\sqrt{\frac{t+n}{r_{k}}}\right\}\right]. \numberthis \label{e:eigenvalue_additive_lower_bound}
  %  &\ge s_k \left[1-\frac{2c_1(t+n)}{r_{k}}\right]
\end{align*}
Also, the integral term $\spadesuit$ is positive, because $c_4$ is chosen to be small enough, $\eta$ scales linearly with $n$, and $R_k\ge r_k\ge bn$ for suitably large $b$. Thus,
\begin{align*}
    \spadesuit & = \int_{\frac{1}{s_k\left[1-c_3 \left(\frac{n+\eta}{r_{k}}+\sqrt{\frac{n+\eta}{R_{k}}}\right)\right]}}^{\frac{1}{c_4 s_k}}\Pr\left[\mu_n(T)\le \frac{1}{\omega}\right] \; \mathrm{d} \omega \\
    & \le \int_{\frac{1}{s_k\left[1-c_5 \sqrt{\frac{n+\eta}{r_{k}}}\right]}}^{\frac{1}{c_4 s_k}}\Pr\left[\mu_n(T)\le \frac{1}{\omega}\right] \; \mathrm{d} \omega \\
    %& \hspace{1in}\mbox{(since $R_{k}\ge r_{k}\ge bn$ and because $\eta$ will be chosen to scale with linearly $n$)}\\
    & \le 2 \int_{\frac{1}{s_k \left[1-c_5 \sqrt{\frac{n+\eta}{r_{k}}}\right]}}^{\frac{1}{c_4 s_k}}\exp\left[-r_{k}\min\left\{\frac{\left(1-\frac{1}{\omega s_k }\right)}{2c_3},\frac{\left(1-\frac{1}{\omega s_k }\right)^2}{4c_3^2}\right\}+n\right]\; \mathrm{d} \omega \\
&    \hspace{0.1in}\mbox{(applying inequality~\eqref{e:eigenvalue_additive_lower_bound}, and by setting $1/\omega$ equal to the RHS of \eqref{e:eigenvalue_additive_lower_bound} and solving for $t$)}\\
& \le 2 e^n\int_{\frac{1}{s_k \left[1-c_5 \sqrt{\frac{n+\eta}{r_{k}}}\right]}}^{\frac{1}{c_4 s_k}}\exp\left[-c_6r_{k}\left(1-\frac{1}{\omega s_k }\right)^2\right]\; \mathrm{d} \omega,
\end{align*}
where the last inequality follows since $\omega>1/s_k$ and therefore the term in the round bracket is always smaller than $1$. Thus, we get that
\begin{align*}
    \spadesuit & \le 2e^n\int_{\frac{1}{s_k \left[1-c_5 \sqrt{\frac{n+\eta}{r_{k}}}\right]}}^{\frac{1}{c_4 s_k}}\exp\left[-c_6r_{k}\left(1-\frac{1}{\omega s_k }\right)^2\right]\; \mathrm{d} \omega .
    \end{align*}
    Now we set $\eta = c_7 n$, for a large enough constant $c_7$, and
    perform a change of variables, redefining $ 1-\frac{1}{\omega s_k} \rightarrow \bar{\omega}$, to get
    \begin{align*}
\spadesuit& \le \frac{2e^n}{s_k}\int_{c_5 \sqrt{\frac{(c_7 + 1) n}{r_{k}}}}^{1-c_4}\frac{\exp\left(-c_6r_{k}\bar{\omega}^2\right)}{(1-\bar{\omega})^2}\; \mathrm{d} \bar{\omega} \\
        & \le \frac{2\exp\left(-c_8n\right)}{s_k}\int_{c_5 \sqrt{\frac{(c_7 + 1) n}{r_{k}}}}^{1-c_4}\frac{1}{(1-\bar{\omega})^2}\; \mathrm{d} \omega \\
        & = \frac{2\exp\left(-c_8n\right)}{s_k}
        \left[\frac{1}{1 - c_5 \sqrt{\frac{(c_7 + 1) n}{r_{k}}}}-\frac{1}{c_4}\right]\\
        &\overset{(i)}{\le} 
        \frac{c_9\exp\left(-c_8 n\right)}{s_k}\label{e:spadesuit_bound_final}\numberthis,
    \end{align*}
    where $(i)$ holds because
    $r_k \geq b n$ for a large value of $b$.

Finally, we turn our attention to the term $\vardiamond$. By using Lemma~\ref{l:smallest_singular_value} we know that
\begin{align*}
    \vardiamond & = \int_{\frac{1}{c_4 s_k}}^{\infty} \Pr\left[\mu_n(T)\le \frac{1}{\omega}\right] \; \mathrm{d} \omega \\
    & \le \int_{\frac{1}{c_4 s_k}}^{\infty} \left(\frac{c_{10}}{\omega s_k}\right)^{c_{11} r_{k}} \; \mathrm{d} \omega \\
    & = \frac{1}{c_4 s_k (c_{11} r_k -1)} \left(c_4 c_{10} \right)^{c_{11} r_{k}}  \le \frac{c_{12}}{r_{k}s_k}, \numberthis \label{e:diamondsuit_expectation_bound}
\end{align*}
where the last inequality follows since $r_{k} \ge b n$ and because $c_4$ is chosen to be small enough. 

By combining inequalities~\eqref{e:decomposition_into_suits},~\eqref{e:clubsuit_expected_value_upper_bound},~\eqref{e:spadesuit_bound_final} and~\eqref{e:diamondsuit_expectation_bound} we conclude that 
\begin{align*}
   & \E\left[\frac{1}{\mu_n(T)}\right] \\& \le \frac{1}{s_k \left[1-c_{3}\left(\frac{n+\eta}{r_{k}}+ \sqrt{\frac{n+\eta}{R_{k}}}\right)\right]}+\frac{c_9\exp\left(-c_8 n\right)}{s_k}+\frac{c_{12}}{r_{k}s_k}\\
   & \le \frac{1}{s_k \left[1-c_{13}\left(\frac{n}{r_{k}}+ \sqrt{\frac{n}{R_{k}}}\right)\right]}+\frac{c_9\exp\left(-c_8 n\right)}{s_k}+\frac{c_{12}}{r_{k}s_k}\hspace{0.3in}\mbox{(since $\eta = c_7 n$)}\\
   & = \frac{1}{s_k}\left[1+c_{14}\left(\frac{\sqrt{\frac{n}{R_{k}}}+\frac{n}{r_{k}}}{1-c_{13}\left(\sqrt{\frac{n}{R_{k}}}+\frac{n}{r_{k}}\right)} +\exp(-c_8n)+\frac{1}{r_{k}}\right)\right]\\
   & \le \frac{1}{s_k}\left[1+c_{15}\left(\sqrt{\frac{n}{R_{k}}}+\frac{n}{r_{k}}+ \exp(-c_8n)\right)\right],
\end{align*}
where the last inequality follows since $R_{k}\ge r_{k}\ge bn$ with $b$ being large enough. 
Hence by inequality~\eqref{e:upper_bound_expected_trace}
\begin{align*}
    \E\left[\Tr(T^{-1})\right] & \le \frac{n}{s_k}\left[1+c_{15}\left(\sqrt{\frac{n}{R_{k}}}+\frac{n}{r_{k}}+\exp(-c_8 
    n)\right)\right]\\
    & \le \frac{n}{s_k}\left[1+c_{15}\left(\sqrt{\frac{n}{R_{k}}}+\frac{n}{r_{k}}+\exp(-c_{8}\sqrt{n})\right)\right],
\end{align*}
which combined with inequality~\eqref{e:expected_value_of_A_sandwich_inequality} completes our proof.
\end{proof}

\subsubsection{Proof of Lemma~\ref{l:concentration_of_Tr_A_inverse}}
As mentioned previously, by using the previous four lemmas we will now show that the trace of $A^{-1}$ is close to $n/s_k$ with high probability. Recall the statement of the lemma from above.
\concentrationofTrAinverse*
\begin{proof}Recall that by definition $A= XX^{\top}$. By an application of the triangle inequality,
\begin{align}\nonumber
      \left|\Tr(A^{-1})-\frac{n}{s_k}\right| & \le   \left|\Tr(A^{-1})-\Tr(T^{-1})\right|+  \left|\E\left[\Tr(A^{-1})\right]-\E\left[\Tr(T^{-1})\right]\right| \\ &\qquad +\left|\Tr(T^{-1})-\E\left[\Tr(T^{-1})\right]\right|+\left|\E\left[\Tr(A^{-1})\right]-\frac{n}{s_k}\right|. \label{e:triangle_inequality_concentration_A}
\end{align}
By Lemma~\ref{l:trace_A_close_to_trace_T} we know that
\begin{align}
    \left|\Tr(A^{-1})-\Tr(T^{-1})\right| & \le \frac{c_5 k}{ s_k} \label{e:first_piece_concentration_A}
\end{align}
with probability at least 
\begin{align*}
    1-2\exp(-c_6 r_{k})- (c_7)^{c_8 r_{k}}\ge 1-c_9\exp(-c_{10} r_{k})\ge 1-c_9\exp(-c_{11} n),
\end{align*}
where the last two inequalities follow since $r_{k}\ge bn$ for some large enough value of $b$. 
Next, by Lemma~\ref{l:expectations_are_close} we know that
\begin{align}
    \left|\E\left[\Tr(A^{-1})\right]-\E\left[\Tr(T^{-1})\right]\right| &\le \frac{c_{12} k}{s_k}.\label{e:second_piece_concentration_A}
\end{align}
By Lemma~\ref{l:trace_T_close_to_its_expectation} we get that with probability at least $1-2e^{-n}$,
\begin{align}
    \left|\Tr(T^{-1})-\E\left[\Tr(T^{-1})\right]\right| & \le \frac{c_{13}n}{s_k}\left[ \sqrt{\frac{n}{R_{k}}}+\frac{n}{r_{k}}\right]. \label{e:third_piece_concentration_A}
\end{align}
Finally, by Lemma~\ref{l:upper_lower_bounds_on_expectation_of_Tr_A_inverse} we know that
\begin{align}
     \left|\E\left[\Tr(A^{-1})\right]-\frac{n}{s_k}\right| &\le \frac{c_{14}n}{s_k}\left[\sqrt{\frac{n}{R_{k}}}+\frac{n}{r_{k}}+\frac{k}{n}+e^{-c_{15}\sqrt{n}}\right].\label{e:fourth_piece_concentration_A}
\end{align}
 Combining the \eqref{e:triangle_inequality_concentration_A}-\eqref{e:fourth_piece_concentration_A} establishes our claim. 
\end{proof}

\subsection{Proof of Lemma~\ref{l:bound_on_alpha}} \label{s:alpha_bound}

Armed with Lemmas~\ref{l:eigenvalue_bound}, \ref{l:smallest_singular_value} and \ref{l:concentration_of_Tr_A_inverse}, 
we are ready to prove Lemma~\ref{l:bound_on_alpha} and establish upper and lower bounds on $\alpha^\star$.
This proof is further divided into a series
of lemmas. 

We prove
bounds on $\alpha^\star$ in terms
of $\lv \tby \rv$ and $\lv w \rv$
in Lemma~\ref{l:alpha_sandwich_inequality}. We in turn
bound $\lv \tby \rv$ in terms of
$\lv \theta^\star \rv$ and
$\lv D^{\dagger}U^{\top}\beps\rv$
in Lemma~\ref{l:tilde_b_y_sandwich}. Next, in Lemma~\ref{l:bound_on_d_u_eps}
we show that, with high probability,
$\lv D^{\dagger}U^{\top}\beps\rv$
is close to
$\sigma^2\Tr\left((XX^{\top})^{-1}\right)$. Recall that, in Section~\ref{ss:concentration_of_Tr_A_inverse}, we showed
that $\Tr\left((XX^{\top})^{-1}\right)$ concentrates around $n/s_k$.

% The proofs of these lemmas are provided in Appendix~\ref{a:alpha_star_appendix}
The next lemma provides an upper and lower bound on $\alpha^\star$.
%\begin{restatable}{lem}{sandwichalpha}\label{l:alpha_sandwich_inequality}
\begin{lemma}
\label{l:alpha_sandwich_inequality}
The scaling factor $\alpha^\star$ satisfies the following
\begin{align*}
    \frac{2\lv \tby\rv^{1/2}}{3}\le \alpha^\star \le \frac{2\lv \tby\rv^{1/2}}{3}\sqrt{\sqrt{1+\frac{4\lv w\rv^4}{81 \lv\tby\rv^2}}+\frac{2\lv w\rv^2}{9\lv\tby\rv}}.
\end{align*}
\end{lemma}
\begin{proof}
Recall the definition of $\alpha^\star$ from above
\begin{align*}
    \alpha^\star = \sqrt{\frac{8\lv\tw_{n+1:p}\rv^2+\sqrt{64\lv\tw_{n+1:p}\rv^4+1296\lv\tby\rv^2}}{81}}.
\end{align*}
Note that $ \lv \tw_{n+1:p}\rv\ge 0$. This immediately leads to the lower bound. For the upper bound note that $ \lv \tw_{n+1:p}\rv \le  \lv \tw\rv  =  \lv V^{\top}w\rv = \lv w\rv$, since $V$ is a unitary matrix.
\end{proof}
% \end{restatable}

The following lemma provides high probability upper and lower bounds on the norm of $\tby$.
\begin{lemma}\label{l:tilde_b_y_sandwich}
The squared norm of $\tby$ satisfies the following
\begin{align*}
  \lv D^{\dagger}U^{\top}\beps\rv^2\left(1-\frac{2\lv \theta^\star\rv}{\lv D^{\dagger}U^{\top}\beps\rv}\right) \le \lv \tby\rv^2 \le \lv D^{\dagger}U^{\top}\beps\rv^2\left(1+\frac{\lv \theta^\star \rv}{\lv D^{\dagger}U^{\top}\beps\rv}\right)^2.
\end{align*}
\end{lemma}
\begin{proof}
Recall that 
$UDV^{\top}$
is the SVD of $X$, 
$\tby = D^{\dagger}U^{\top}\by$ and that $\by = X\theta^\star +\beps$. Therefore 
\begin{align*}
    \tby = D^{\dagger}U^{\top}(X\theta^\star +\beps)
     = D^{\dagger}U^{\top}(UDV^{\top}\theta^\star +\beps)
    & = D^{\dagger}DV^{\top}\theta^\star +D^{\dagger}U^{\top}\beps\\
    & = \begin{bmatrix}I_n \\
    0_{(p-n)\times n}\end{bmatrix}V^{\top}\theta^\star +D^{\dagger}U^{\top}\beps.
\end{align*}
Define $\widetilde{\theta}^\star := V^{\top}\theta^\star$ and so $$\tby  = \widetilde{\theta}^\star_{1:n}+ D^{\dagger}U^{\top}\beps.$$
Thus,
\begin{align*}
    \lv \tby\rv^2 & = \lv \widetilde{\theta}^\star_{1:n} \rv^2+ \lv D^{\dagger}U^{\top}\beps\rv^2 +2 \left(\beps^{\top}UD^{\dagger \top}\right)\left(\widetilde{\theta}^\star_{1:n}\right).
\end{align*}
Now since $0\le \lv \widetilde{\theta}^\star_{1:n} \rv  \le \lv \widetilde{\theta}^\star \rv  = \lv V^{\top}\theta^\star\rv= \lv \theta^\star\rv$ we get that
\begin{align*}
    \lv \tby\rv^2 \ge \lv D^{\dagger}U^{\top}\beps\rv^2 - 2\lv D^{\dagger}U^{\top}\beps\rv\lv \theta^\star\rv & = \lv D^{\dagger}U^{\top}\beps\rv^2\left(1-\frac{2\lv \theta^\star\rv}{\lv D^{\dagger}U^{\top}\beps\rv}\right)
\end{align*}
and also that
\begin{align*}
     \lv \tby\rv^2 & \le \lv D^{\dagger}U^{\top}\beps\rv^2 + 2\lv D^{\dagger}U^{\top}\beps\rv\lv \theta^\star\rv + \lv \theta^\star \rv^2 \\
     & = \left(\lv D^{\dagger}U^{\top}\beps\rv+\lv \theta^\star \rv\right)^2 = \lv D^{\dagger}U^{\top}\beps\rv^2\left(1+\frac{\lv \theta^\star \rv}{\lv D^{\dagger}U^{\top}\beps\rv}\right)^2,
\end{align*}
which establishes our claim.
\end{proof}

The next result upper and lower bounds $\lv D^{\dagger}U^{\top}\beps\rv^2$ with high probability.
\begin{lemma}
\label{l:bound_on_d_u_eps}
For any $t\ge 0$, with probability at least $1-2e^{-t}$
\begin{align*}
    \left|\lv D^{\dagger}U^{\top}\beps\rv^2 -\sigma^2\Tr\left((XX^{\top})^{-1}\right)\right|& \le c\max\left\{\frac{t}{\mu_n(XX^{\top})},\sqrt{t\sum_{i=1}^n \frac{1}{\mu_i^2(XX^{\top})}}\right\} .
\end{align*}
\end{lemma}
\begin{proof} Let $u_1,\ldots,u_n$ be the columns of $U$. The matrix $U$ is unitary so each column $u_i$ has unit norm. So
\begin{align*}
    \lv D^{\dagger}U^{\top}\beps\rv^2 & = \sum_{i=1}^n \frac{(u_i^{\top}\beps)^2}{D_{ii}^2} = \sum_{i=1}^n \frac{(u_i^{\top}\beps)^2}{\mu_i(XX^{\top})} 
\end{align*}
and 
\begin{align*}
        \E_{\beps}\left[\lv D^{\dagger}U^{\top}\beps\rv^2 \mid X\right] & = \sum_{i=1}^n \frac{\E\left[(u_i^{\top}\beps)^2\mid X\right]}{D_{ii}^2} =\sum_{i=1}^n \frac{\sigma^2}{D_{ii}^2}= \sigma^2\Tr\left((XX^{\top})^{-1}\right).
\end{align*}
Since the components are $\beps$ are independent, $\sigma_y^2$-sub-Gaussian random variables, with variance $\sigma^2$, by invoking the Hanson-Wright inequality~\citep[see][Theorem~1]{rudelson2013hanson} we infer that 
\begin{align*}
    \left|\lv D^{\dagger}U^{\top}\beps\rv^2 -\sigma^2\Tr((XX^{\top})^{-1})\right| &= \left|\beps^{\top}\left(UD^{\dagger \top}D^\dagger U^\top\right) \beps  -\sigma^2\Tr((XX^{\top})^{-1})\right| \\
    &\le c_1\sigma_y^2\max\left\{\frac{t}{\mu_n(XX^{\top})},\sqrt{t\cdot \sum_{i=1}^n \frac{1}{\mu_i^2(XX^{\top})}}\right\} \\
    & = c\max\left\{\frac{t}{\mu_n(XX^{\top})},\sqrt{t\cdot\sum_{i=1}^n \frac{1}{\mu_i^2(XX^{\top})}}\right\} 
\end{align*}
with probability at least $1-2e^{-t}$, completing the proof.
\end{proof}

With these lemmas in place we are now ready to prove Lemma~\ref{l:bound_on_alpha}. We restate it here.
\boundonalpha*
\begin{proof}
Using Lemma~\ref{l:concentration_of_Tr_A_inverse},  
% we get that, 
with probability at least $1-c_6e^{-c_7 n}$,
\begin{align*}
    \left|\Tr(XX^{-1})-\frac{n}{s_k}\right|
    & \le \frac{c_8 n}{s_k}  \left[  \sqrt{\frac{n}{R_{k}}}+\frac{n}{r_{k}}+\frac{ k}{n} +e^{-c_9\sqrt{n}}\right]. \numberthis \label{e:trace_of_xx_transpose_absolute_bound}
\end{align*}
Next, by Lemma~\ref{l:eigenvalue_bound},
%we get that 
with probability at least $1-2e^{-\sqrt{n}}$, for all $i\in [n]$
\begin{align*}
    \mu_i(XX^{\top})&\ge s_k\left[1-c_{10}\left(\frac{n+\sqrt{n}}{r_{k}}+\sqrt{\frac{n+\sqrt{n}}{R_{k}}}\right)\right]\\
    &\ge s_k\left[1-c_{11}\left(\frac{n}{r_{k}}+\sqrt{\frac{n}{r_{k}}}\right)\right]\hspace{0.5in}\mbox{(since $R_k\ge r_k$ by Lemma~\ref{l:effective_ranks})}\\ &\ge c_{12}s_k \hspace{1in} \mbox{(since $r_k\ge bn$)}.
\end{align*}
This, combined with Lemma~\ref{l:bound_on_d_u_eps}, tells us that for any $\delta \in (e^{-c_0\sqrt{n}},1)$ with probability at least $1-c_{13}\delta$
\begin{align*}
    \left|\lv D^{\dagger}U^{\top}\beps\rv^2 -\sigma^2\Tr((XX^{\top})^{-1})\right| &\le c_{14}\max\left\{
  %         \frac{\sqrt{\log(2/\delta)}}{s_k},
           \frac{\log(2/\delta)}{s_k},
                     \frac{\sqrt{n \log(2/\delta)}}{s_k}\right\} \\
    & \le \frac{c_{14}\sqrt{n \log(2/\delta)}}{s_k}.
\end{align*}
Combining this with inequality~\eqref{e:trace_of_xx_transpose_absolute_bound} and
recalling that $\sigma^2$ is a constant, we infer that,
with probability at least
$1-c_3\delta$,
\begin{align}\nonumber
    \left|\lv D^{\dagger}U^{\top}\beps\rv^2 -\frac{\sigma^2 n}{s_k}\right| &\le \frac{c_{15}n}{s_k}\left[\sqrt{\frac{n}{R_{k}}}+\frac{n}{r_{k}}+e^{-c_9\sqrt{n}} + \sqrt{\frac{\log(2/\delta)}{n}}+\frac{ k}{n}\right]\\
    &\le  \frac{c_{16}n}{s_k}\left[\sqrt{\frac{n}{R_{k}}}+\frac{n}{r_{k}}+ \sqrt{\frac{\log(2/\delta)}{n}}+\frac{ k}{n}\right] \label{e:d_u_eps_dagger_bound}\\
    &\le  \frac{c_{17}n}{s_k}, \label{e:d_u_eps_dagger_bound_crude}
\end{align}
where the last inequality follows since $R_{k}\ge r_{k}\ge bn$, $n\ge c_2 k$ and since $\delta \ge e^{-c_0\sqrt{n}}$. 

We shall assume that condition~\eqref{e:d_u_eps_dagger_bound} holds going forward. (This determines the success probability in the statement of the lemma.) Now since $r_{k}\ge bn$ and $n\ge c_2\max\{k,s_k\}$ for a large enough constants $b$ and $c_2$, by invoking Lemma~\ref{l:tilde_b_y_sandwich}, we find that
\begin{align*}
    \lv \tby\rv^2 &\le \lv D^{\dagger}U^{\top}\beps\rv^2\left(1+\frac{\lv \theta^\star \rv}{\lv D^{\dagger}U^{\top}\beps\rv}\right)^2 \\
    &= \lv D^{\dagger}U^{\top}\beps\rv^2\left(1+\frac{\lv \theta^\star \rv^2}{\lv D^{\dagger}U^{\top}\beps\rv^2}+\frac{2\lv \theta^\star \rv}{\lv D^{\dagger}U^{\top}\beps\rv}\right)\\
    &\overset{(i)}{\le} \frac{\sigma^2 n}{s_k} \left[1+c_{16}\left(\sqrt{\frac{n}{R_{k}}}+\frac{n}{r_{k}}+ \sqrt{\frac{\log(2/\delta)}{n}}+\frac{ k}{n} \right)\right]\\&\hspace{1in}\times \left(1+c_{18}\left(\frac{\lv \theta^\star\rv^2s_k}{n}+\frac{\lv \theta^\star\rv\sqrt{s_k}}{\sqrt{n}}\right)\right)\\
    &\overset{(ii)}{\le} \frac{\sigma^2 n}{s_k} \left[1+c_{16}\left(\sqrt{\frac{n}{R_{k}}}+\frac{n}{r_{k}} + \sqrt{\frac{\log(2/\delta)}{n}}+\frac{ k}{n}\right)\right] \left(1+c_{19}\sqrt{\frac{s_k}{n}}\right) \label{e:y_tilde_upper_sharp}\numberthis\\
    &\le \frac{c_{20}n}{s_k} \label{e:y_tilde_upper_crude}\numberthis,
\end{align*}
where $(i)$ follows by applying inequalities \eqref{e:d_u_eps_dagger_bound} and \eqref{e:d_u_eps_dagger_bound_crude},  and also because $\sigma^2$ is a constant. The second inequality $(ii)$ follows since $\lv \theta^\star\rv\le c_3$ and because $n\ge c_2 s_k$.
Also, by Lemma~\ref{l:tilde_b_y_sandwich}, we get that
\begin{align*}
\lv \tby\rv^2 &\ge   \lv D^{\dagger}U^{\top}\beps\rv^2\left(1-\frac{2\lv \theta^\star\rv}{\lv D^{\dagger}U^{\top}\beps\rv}\right)  \\
&\ge \frac{\sigma^2 n}{s_k} \left[1-c_{16}\left(\sqrt{\frac{n}{R_{k}}}+\frac{n}{r_{k}} +\sqrt{\frac{\log(2/\delta)}{n}}+\frac{ k}{n}\right)\right] \left(1-c_{21}\sqrt{\frac{s_k}{n}}\right)\label{e:y_tilde_lower_sharp}\numberthis\\
&\ge \frac{c_{22}n}{s_k}, \label{e:y_tilde_lower_crude}\numberthis
\end{align*}
where the last two inequalities follow by repeating the logic from the previous equation block.

Now recall that, by Lemma~\ref{l:alpha_sandwich_inequality},
\begin{align}
     \frac{2\lv \tby\rv^{1/2}}{3}\le \alpha^\star \le \frac{2\lv \tby\rv^{1/2}}{3}\sqrt{\sqrt{1+\frac{4\lv w\rv^4}{81 \lv\tby\rv^2}}+\frac{2\lv w\rv^2}{9\lv\tby\rv}}. \label{e:alpha_upper_lower_sandwich_bound}
\end{align}
Using the lower bound in the equation above combining with inequality~\eqref{e:y_tilde_lower_sharp} we find that
\begin{align*}
  \alpha^\star &\ge   \frac{2\sqrt{\sigma} n^{1/4}}{3s_k^{1/4}} \left[1-c_{16}\left(\sqrt{\frac{n}{R_{k}}}+\frac{n}{r_{k}} +\sqrt{\frac{\log(2/\delta)}{n}}+\frac{ k}{n}\right)\right] \left(1-c_{23}\sqrt{\frac{s_k}{n}}\right),
\end{align*}
and since $n\ge c_2s_k$ for a large enough constant $c_2$,
\begin{align*}
    \frac{\alpha^\star}{ \frac{2\sqrt{\sigma} n^{1/4}}{3s_k^{1/4}}}-1 \ge -c_{24}\left[\sqrt{\frac{n}{R_{k}}}+\frac{n}{r_{k}}+\sqrt{\frac{s_k}{n}}+ \sqrt{\frac{\log(2/\delta)}{n}}+\frac{ k}{n} \right]. \label{e:alpha_lower_bound}\numberthis
\end{align*}
Now for the upper bound, since $\lv w\rv \le c_3$, by using \eqref{e:y_tilde_lower_crude} we have that $\lv w\rv^2/\lv \tby \rv \le 1/20$, since $n>c_2s_k$, where $c_2$ is large enough. Thus, by \eqref{e:alpha_upper_lower_sandwich_bound},
\begin{align*}
    \alpha^\star &\le \frac{2\lv \tby\rv^{1/2}}{3}\left(1+\frac{c_{25}\lv w\rv}{\lv \tby\rv^{1/2}}\right)\\
    &\le \frac{2\sqrt{\sigma} n^{1/4}}{3s_k^{1/4}} \left[1+c_{26}\left(\sqrt{\frac{n}{R_{k}}}+\frac{n}{r_{k}} + \sqrt{\frac{\log(2/\delta)}{n}}+\frac{ k}{n}\right)\right] \left(1+c_{27}\sqrt{\frac{s_k}{n}}\right)
\end{align*}
and so
\begin{align*}
     \frac{\alpha^\star}{ \frac{2\sqrt{\sigma} n^{1/4}}{3s_k^{1/4}}}-1 \le c_{28}\left[\sqrt{\frac{n}{R_{k}}}+\frac{n}{r_{k}}+\sqrt{\frac{s_k}{n}}+ \sqrt{\frac{\log(2/\delta)}{n}}+\frac{ k}{n} \right].
\end{align*}
This combined with \eqref{e:alpha_lower_bound} completes the proof.
\end{proof}
\subsection{Proof of Theorem~\ref{t:main}}\label{ss:main_theorem_proof}
Let us first restate the theorem.
\main*
\begin{proof}
By Lemma~\ref{l:excess_risk_decomposition}, we know that 
\begin{align*}
    \mathsf{Risk}(\htheta) & \le c_8 (\theta^\star - \alpha^\star w)^{\top}B(\theta^\star - \alpha^\star w)+c_8\log(1/\delta)\Tr(C)
\end{align*}
with probability at least $1-\delta$, where the matrices
\begin{align*}
    B &= \left(I - X^{\top}(XX^{\top})^{-1}X\right)\Sigma\left(I - X^{\top}(XX^{\top})^{-1}X\right) \quad \text{and} \\
    C &= (XX^{\top})^{-1}X\Sigma X^{\top}(XX^{\top})^{-1}.
\end{align*}
Recall that $\psi = \frac{2\sqrt{\sigma}n^{1/4}w}{3s_k^{1/4}}$. Thus, with the same probability
\begin{align*}
     \mathsf{Risk}(\htheta) & \le c_8 \left(\theta^\star -\psi - \left(\alpha^\star w-\psi\right)\right)^{\top}B\left(\theta^\star -\psi - \left(\alpha^\star w-\psi\right)\right)+c_8\log(1/\delta)\Tr(C)\\
     & = c_8 \lv \theta^\star -\psi - \left(\alpha^\star w-\psi\right)\rv_{B}^2+c_8\log(1/\delta)\Tr(C)\\
     & \le 2c_8 \lv \theta^\star -\psi\rv_B^2 + 2c_8 \lv\alpha^\star w-\psi\rv_{B}^2+c_8\log(1/\delta)\Tr(C)\\
     &= \underbrace{2c_8 \left(\theta^\star -\psi\right)^{\top}B\left(\theta^\star -\psi\right)}_{``\mathsf{Bias}"}+\underbrace{c_8\log(1/\delta)\Tr(C)}_{``\mathsf{Variance}"} + \underbrace{2c_8\lv B\rv_{op}\lv \alpha^\star w-\psi \rv^2}_{``\mathsf{\Xi}"}. \numberthis \label{e:risk_decomposition_into_three_terms}
\end{align*}

We shall bound each of the three terms in the inequality above to establish the theorem.

 Recall the definition of the matrix $T=\sum_{j>k}\lambda_j z_j z_j^{\top}$ from Definition~\ref{def:A_head_tail} above. Define $S := \{j: j>k\}$, and let $X_{S} \in \mathbb{R}^{n\times |S|}$ be the submatrix formed by the last $p-k$ columns of $X \in \R^{n\times p}$. It can be verified that $T=X_SX_S^{\top}$. By Lemma~\ref{l:eigenvalue_bound}, with probability at least $1-2e^{-n}\ge 1-c_{9}\delta$, (since $\delta \ge e^{-c_0\sqrt{n}}$)
\begin{align*}
    \mu_1(T)\le c_{10}\sum_{j>k}\lambda_j \quad \text{and} \quad  \mu_n(T)\ge c_{11}\sum_{j>k}\lambda_j.
\end{align*}
Therefore, the condition number of the matrix $T$ is a constant with the same probability. Assuming this bound on the condition number holds we shall bound the first two terms in \eqref{e:risk_decomposition_into_three_terms}.

\textit{Bound on the bias and variance:} Since the condition number of $T$ is at most a constant, by invoking \citep[][Theorem~1]{tsigler2020benign} we get that with probability at least $1-c_{12}\delta$
\begin{align}\label{e:bias_bound}
    \mathsf{Bias}& \le c_7\left(\lv(\theta^\star-\psi)_{1:k} \rv_{\Sigma^{-1}_{1:k}}^2\left(\frac{s_k}{n}\right)^2+\lv (\theta^\star-\psi)_{k+1:p}\rv_{\Sigma_{k+1:p}}^2  \right)
\end{align}
and
\begin{align}\label{e:variance_bound}
    \mathsf{Variance}& \le c_7 \log(1/\delta)\left(\frac{k}{n}+\frac{n}{R_{k}}  \right).
\end{align}
%Added additional logic below.
We simplify our upper bound on $\mathsf{Bias}$ by noting that under our choice of $k$ as follows:
\begin{align*}
    &c_7\left(\lv(\theta^\star-\psi)_{1:k} \rv_{\Sigma^{-1}_{1:k}}^2\left(\frac{s_k}{n}\right)^2+\lv (\theta^\star-\psi)_{k+1:p}\rv_{\Sigma_{k+1:p}}^2 \right) \\
    &\qquad  =c_7 \sum_{i=1}^p \left[\mathbb{I}(i\le k) (\theta^\star_i - \psi_i)^2 \frac{s_k^2}{n^2 \lambda_i}+\mathbb{I}(i> k) \lambda_i (\theta^\star_i -\psi_i)^2 \right] \\
    &\qquad  =c_7 \sum_{i=1}^p \lambda_i (\theta^\star_i -\psi_i)^2 \left[\mathbb{I}(i\le k)  \frac{s_k^2}{n^2 \lambda_i^2}+\mathbb{I}(i> k)  \right] \\
    &\qquad  =c_7 \sum_{i=1}^p \lambda_i (\theta^\star_i -\psi_i)^2 \frac{(\frac{s_k}{n})^2}{(\frac{s_k}{n})^2 + \lambda_i^2} \left[\mathbb{I}(i\le k)  \left(1+ \frac{1}{\lambda_i^2}\left(\frac{s_k}{n}\right)^2\right)+\mathbb{I}(i> k)\left(1+ {\lambda_i^2}\left(\frac{n}{s_k}\right)^2\right)  \right] \\
    &\qquad  \overset{(i)}{\le} c_7 \sum_{i=1}^p \lambda_i (\theta^\star_i -\psi_i)^2 \frac{(\frac{s_k}{n})^2}{(\frac{s_k}{n})^2 + \lambda_i^2} \left[\mathbb{I}(i\le k)  \left(1+ b^2\right)+\mathbb{I}(i> k)\left(1+ {\lambda_{i}^2}\left(\frac{n}{s_k}\right)^2\right)  \right] \\
      &\qquad \le  c_7 \sum_{i=1}^p \lambda_i (\theta^\star_i -\psi_i)^2 \frac{(\frac{s_k}{n})^2}{(\frac{s_k}{n})^2 + \lambda_i^2} \left[\mathbb{I}(i\le k)  \left(1+ b^2\right)+\mathbb{I}(i> k)\left(1+ {\lambda_{k+1}^2}\left(\frac{n}{s_k}\right)^2\right)  \right] \\
    &\qquad  \le c_7 \sum_{i=1}^p \lambda_i (\theta^\star_i -\psi_i)^2 \frac{(\frac{s_k}{n})^2}{(\frac{s_k}{n})^2 + \lambda_i^2} \left[\mathbb{I}(i\le k)  \left(1+ b^2\right)+\mathbb{I}(i> k)\left(1+ \left(\frac{n}{r_k}\right)^2\right)  \right] \\
     &\qquad  \overset{(ii)}{\le} c_7 \sum_{i=1}^p \lambda_i (\theta^\star_i -\psi_i)^2 \frac{(\frac{s_k}{n})^2}{(\frac{s_k}{n})^2 + \lambda_i^2} \left[\mathbb{I}(i\le k)  \left(1+ b^2\right)+\mathbb{I}(i> k)\left(1+ \frac{1}{b^2}\right)  \right] \\
     &\qquad  \le c_{13} \sum_{i=1}^p \lambda_i (\theta^\star_i -\psi_i)^2 \frac{(\frac{s_k}{n})^2}{(\frac{s_k}{n})^2 + \lambda_i^2}, 
\end{align*}
where $(i)$ follows since by definition $k = \min\{j\ge 0 : r_j \ge bn\}$ and so for $i\le k$, $s_k/\lambda_i \le s_i/\lambda_i = r_i < bn$. Inequality~$(ii)$ follows since $r_k \ge bn$. Continuing we get that
\begin{align*}
    &c_7\left(\lv(\theta^\star-\psi)_{1:k} \rv_{\Sigma^{-1}_{1:k}}^2\left(\frac{s_k}{n}\right)^2+\lv (\theta^\star-\psi)_{k+1:p}\rv_{\Sigma_{k+1:p}}^2 \right) \\
    &\qquad \le c_{13} \sum_{i=1}^p \lambda_i (\theta^\star-\psi)_i^{2} \frac{\left(\frac{s_k}{n}\right)^2}{\left(\frac{s_k}{n}\right)^2+\lambda_i^2}\\
    &\qquad = c_{13}\left(\frac{s_k}{n}\right)^2 \sum_{i=1}^p  (\theta^\star-\psi)_i^{2} \frac{\lambda_i}{\left(\frac{s_k}{n}\right)^2+\lambda_i^2} \\
    &\qquad \le c_{13}\left(\frac{s_k}{n}\right)^2 \lv \theta^\star - \psi\rv^2 \max_{i \in [p]} \frac{\lambda_i}{\left(\frac{s_k}{n}\right)^2+\lambda_i^2}\hspace{1in}\mbox{(by H\"older's inequality)}\\
    &\qquad \le c_{13}\left(\frac{s_k}{n}\right)^2 \lv \theta^\star - \psi\rv^2 \max_{\zeta \ge 0} \frac{\zeta}{\left(\frac{s_k}{n}\right)^2+\zeta^2} =  \frac{ 2c_{13} \lv \theta^\star - \psi\rv^2 s_k}{n}. \numberthis \label{e:bias_weaker_bound_main_theorem}
\end{align*}
\textit{Bound on $\mathsf{\Xi}$ (the estimation error of $\alpha^\star$):} By Lemma~\ref{l:bound_on_alpha} with probability at least $1-c_{14}\delta$
\begin{align*}
     \left|\alpha^\star -\frac{2\sqrt{\sigma}n^{1/4}}{3s_k^{1/4}}\right|\le c_{15} \frac{2\sqrt{\sigma}n^{1/4}}{3s_k^{1/4}}\left[\sqrt{\frac{n}{R_{k}}}+\frac{n}{r_{k}}+\sqrt{\frac{s_k}{n}}+ \sqrt{\frac{\log(2/\delta)}{n}}+\frac{ k}{n} \right]
\end{align*}
and therefore,
\begin{align} \nonumber
    \lv \alpha^\star w - \psi \rv &\le c_{16}\lv \psi\rv\left[\sqrt{\frac{n}{R_{k}}}+\frac{n}{r_{k}}+\sqrt{\frac{s_k}{n}}+ \sqrt{\frac{\log(2/\delta)}{n}}+\frac{ k}{n} \right]\\
    &\le c_{17}\lv \psi\rv\left[\sqrt{\frac{n}{R_{k}}}+\frac{n}{r_{k}}+\sqrt{\frac{s_k}{n}}+ \sqrt{\frac{\log(1/\delta)}{n}}+\frac{ k}{n} \right] \hspace{0.1in}\mbox{(since $\delta \le 1-c_1e^{-c_2n}$)}
    . \label{e:alphaw_minus_psi_bound}
\end{align}
To control the operator norm of $B$, we first observe that
\begin{align*}
    \lv B\rv_{op}&= \left\lv\left(I - X^{\top}(XX^{\top})^{-1}X\right)\Sigma\left(I - X^{\top}(XX^{\top})^{-1}X\right)\right\rv_{op}  \\
    &= \left\lv\left(I - X^{\top}(XX^{\top})^{-1}X\right)\left(\Sigma-\frac{X^{\top}X}{n}\right)\left(I - X^{\top}(XX^{\top})^{-1}X\right) \right\rv_{op} \\
    &\le \left\lv I - X^{\top}(XX^{\top})^{-1}X\right\rv_{op}^2\left\lv\Sigma-\frac{X^{\top}X}{n}\right\rv_{op}\\
    &\le \left\lv\Sigma-\frac{X^{\top}X}{n}\right\rv_{op}.
\end{align*}
Thus, by invoking \citep[][Theorem~9]{koltchinskii2017concentration} we get that with probability at least $1-\delta$
\begin{align}
    \lv B\rv_{op}&\le c_{18}\lambda_1 \max\left\{\sqrt{\frac{r_0}{n}},\frac{r_0}{n},\sqrt{\frac{\log(1/\delta)}{n}},\frac{\log(1/\delta)}{n}\right\}\nonumber\\
    &\le c_{18}\lambda_1 \max\left\{\sqrt{\frac{r_0}{n}},\frac{r_0}{n},\sqrt{\frac{\log(1/\delta)}{n}}\right\} ,\label{e:b_operator_norm_bound}
\end{align}
where the second inequality follows since $\delta \ge e^{-c_0\sqrt{n}}$. 

Combining inequalities~\eqref{e:alphaw_minus_psi_bound} and \eqref{e:b_operator_norm_bound} we get that with probability at least $1-c_{19}\delta$
\begin{align*}
    2 c_{8}\lv B\rv_{op} \lv \alpha^\star w - \psi \rv^2 
    &\le c_{20} \lambda_1\lv \psi\rv^2 \max\left\{\sqrt{\frac{r_0}{n}},\frac{r_0}{n},\sqrt{\frac{\log(1/\delta)}{n}}\right\}\\ &\qquad  \times \left[\frac{n}{R_{k}}+\frac{n^2}{r_{k}^2}+\frac{s_k}{n}+ \frac{\log(2/\delta)}{n}+\frac{ k^2}{n^2} \right]. \label{e:estimation_error_alpha_star} \numberthis
\end{align*}
Combining inequalities \eqref{e:bias_bound}, \eqref{e:variance_bound},  \eqref{e:bias_weaker_bound_main_theorem} and \eqref{e:estimation_error_alpha_star} along with a union bound completes the proof.
\end{proof}

\section{Proof of Proposition~\protect\ref{p:lower}}
\label{s:lower}
Recall the statement of the proposition.
\lowerbound*
\begin{proof}
In the proof of Lemma~\ref{l:BC}, we showed that, for
all $X,\by$, we have
\begin{align*}
 \mathsf{Risk}(\htheta) & = \E_x\left[\left(x^{\top}\left(I-X^{\top}(XX^{\top})^{-1}X\right)(\theta^\star -\alpha^\star w)-x^\top X^\top(XX^\top)^{-1}\beps\right)^2\right].
 \end{align*}
 Expanding the quadratic yields
 \begin{align*}
 \mathsf{Risk}(\htheta) & = 
 \E_x\left[\left(x^{\top}\left(I-X^{\top}(XX^{\top})^{-1}X\right)(\theta^\star -\alpha^\star w)\right)^2\right]+ \E_{x}\left[\left(x^\top X^\top(XX^\top)^{-1}\beps\right)^2\right] \\
  & \hspace{0.5in} 
      + 2 \E_x\left[\left(x^{\top}\left(I-X^{\top}(XX^{\top})^{-1}X\right) \theta^\star \right) \left(x^\top X^\top(XX^\top)^{-1}\beps\right) \right] \\
  & \hspace{0.5in} 
      - 2  \E_x\left[\left(x^{\top}\left(I-X^{\top}(XX^{\top})^{-1}X\right)(\alpha^\star w)\right) \left(x^\top X^\top(XX^\top)^{-1}\beps\right) \right] .
 \end{align*}
Since the distribution of $w$ is symmetric about the origin,
and independent of $X$ and $\by$, and since, for fixed $X$ and $\by$,
$\alpha^\star$ is determined as a function of $w$, after conditioning
on $X$ and $\by$, the distribution of $\alpha^\star w$ is symmetric about
the origin, and therefore has zero mean. This, along with
the fact that $\E[\beps] = 0$, gives
\begin{align*}
 \E[ \mathsf{Risk}(\htheta) ] & = 
  \E\left[\left(x^{\top}\left(I-X^{\top}(XX^{\top})^{-1}X\right)(\theta^\star -\alpha^\star w)\right)^2\right]+ \E_{x}\left[\left(x^\top X^\top(XX^\top)^{-1}\beps\right)^2\right] \\  
  & = \E[(\theta^\star - \alpha^* w)^{\top} B (\theta - \alpha^* w)^\star ]
       + \E\left[\left(x^\top X^\top(XX^\top)^{-1}\beps\right)^2\right] \\
  & \geq \E[\theta^{\star \top} B \theta^\star]
       + \E\left[\left(x^\top X^\top(XX^\top)^{-1}\beps\right)^2\right] \\
  & \geq \E[\theta^{\star \top} B \theta^\star]
       + \sigma^2 \E\left[\Tr(C)\right],
\end{align*}
completing the proof.
\end{proof}

\section{Discussion}\label{s:discussion}

Despite the fact that parameterizing a linear model using a balanced, two-layer
linear network has been shown in previous work to have a substantial
effect on the inductive bias of gradient descent~\citep{azulay2021implicit},
it remains compatible with benign overfitting, and the initial weights also
encode a potentially useful bias.

While Proposition~\ref{p:lower} limits the prospects for improving
our upper bounds, there still appears to be a gap between our upper
and lower bounds.

Moving beyond the case where the initialization is balanced would be an interesting next step. We briefly note that, for the initial parameters to be balanced, it is necessary for the weight matrix in the first layer $W \in \R^{m\times p}$ to have rank one. In the case where there is a single neuron ($m=1$), Theorem~2 by \citep{azulay2021implicit} characterizes the implicit bias of the final solution learnt by gradient flow on the squared loss. The techniques developed in this paper might perhaps be useful in bounding the excess risk of this solution. 

Yet another interesting open question concerns characterizing the implicit bias of gradient flow with the squared loss in the case where a linear model is parameterized using a deeper representation than two layers,
building on existing research \citep{gunasekar2017implicit,gunasekar2018implicit,arora2019implicit,pmlr-v125-woodworth20a,DBLP:conf/iclr/GissinSD20,razin2020implicit,DBLP:conf/iclr/YunKM21,azulay2021implicit,jagadeesan2021inductive}.
It would also be interesting to prove corresponding excess risk bounds for such solutions, and to study the effect of depth on the generalization properties of the resulting models. 

\subsection*{Acknowledgements}
We would like to thank the anonymous reviewers for pointing out a bug in a previous version of the proof of Lemma~\ref{l:smallest_singular_value}, and for their many helpful suggestions. We gratefully acknowledge the support of the NSF through grants DMS-2023505 and DMS-2031883 and of the Simons Foundation through award 814639.

%JMLR
% \acks{}

\appendix
\section{The design matrix has full rank (and more)}
\label{a:full_rank}

\begin{lemma}
\label{l:full_rank}
Under
Assumption~\ref{assumption:small_ball_probability}, for any eigenvector $v$ of $\Sigma$, and any sample size
$n$, the projection of the rows of $X$ onto the subspace
of $\R^p$ orthogonal to $v$ has rank $n$.
\end{lemma}
\begin{proof}
Assume without loss of generality that $\Sigma$ is diagonal and
$v = (1,0,0,\ldots,0)$.  Let $x_1,\ldots,x_n$ denote the rows of
$X$, and let $x_1',\ldots,x_n'$ be obtained from $x_1,\ldots,x_n$
by replacing each of their first components with $0$, thereby projecting
them onto the subspace orthogonal to $v$.  It suffices to prove
that, almost surely, $x_1',\ldots,x_n'$ are linearly independent.

We will prove this by induction.  The base case, there $n=1$,
is straightforward.  When $n > 1$, by the inductive hypothesis,
$x_1',\ldots,x_{n-1}'$ are linearly independent.
Since $p > n$, Assumption~\ref{assumption:small_ball_probability}
implies that the span of 
$x_1',\ldots,x_{n-1}'$ has probability $0$, so that,
almost surely, $x_n'$ is not a member this span, completing
the proof.
\end{proof}

\section{Concentration inequalities}
\label{a:concentration}
For an excellent reference of sub-Gaussian and sub-exponential concentration inequalities we refer the reader to \citet{vershynin2018high}. We begin by defining sub-Gaussian and sub-exponential random variables.

\begin{definition} \label{def:subgaussian}A random variable $\phi$ is sub-Gaussian if 
\begin{align*}
\lv \phi \rv_{\psi_2}:= \inf\left\{t>0: \mathbb{E}[\exp(\phi^2/t^2)]< 2\right\}
\end{align*}
is bounded. Further, $\lv \phi\rv_{\psi_2}$ is defined to be its sub-Gaussian norm.
\end{definition}

\begin{definition} \label{def:subexp}A random variable $\phi$ is 
% to avoid an overall hbox
said to be
sub-exponential if 
\begin{align*}
\lv \phi \rv_{\psi_1}:= \inf\left\{t>0: \mathbb{E}[\exp(\lvert \phi\rvert/t)< 2]\right\}
\end{align*}
is bounded. Further, $\lv \phi\rv_{\psi_1}$ is defined to be its sub-exponential norm.
\end{definition}
Next we state a few well-known facts about sub-Gaussian and sub-exponential random variables.
\begin{lemma}[Vershynin~2018, Lemma~2.7.6]
\label{l:sub_gaussian_squared}If a random variable $\phi$ is sub-Gaussian then $\phi^2$ is sub-exponential with $\lv \phi^2\rv_{\psi_1} = \lv \phi \rv_{\psi_2}^2$.
\end{lemma}

\begin{lemma}[Vershynin~2018, Lemma~2.7.10]
\label{l:sub_exponential_centering}If a random variable $\phi$ is sub-exponential then $\phi-\E[\phi]$ is sub-exponential with $\lv \phi-\E[\phi]\rv_{\psi_1} \le c\lv \phi \rv_{\psi_1}$ for some positive constant $c$.
\end{lemma}
% \begin{lemma}[Vershynin~2018, Theorem~5.2.2]
% \label{l:sub_gaussian_lipschitz}If a random variable $\phi\sim\cN(0,1)$ and $g$ is a $1$-Lipschitz function then $\lv g(\phi)-\E[g(\phi)]\rv_{\psi_{2}}\le c$, for some absolute positive constant $c$.
% \end{lemma}
%  Let us state Hoeffding's inequality \citep[see, e.g.,][Theorem~2.6.2]{vershynin2018high}, a concentration inequality for a sum of independent sub-Gaussian random variables.
% \begin{theorem} \label{thm:hoeffding}
% For independent mean-zero sub-Gaussian random variables $\phi_1,\ldots,\phi_m$, for every $\eta>0$, we have
% \begin{align*}
% \Pr\left[\Big\lvert \sum_{i=1}^m  \phi_i \Big\rvert \ge \eta\right]\le 2\exp\left(- \frac{c\eta^2}{\sum_{i=1}^m \lv \phi_i\rv_{\psi_2}^2}\right),
% \end{align*}
% where $c$ is
% a positive absolute constant.
% \end{theorem}
We state Bernstein's inequality \citep[see, e.g.,][Theorem~2.8.1]{vershynin2018high}, a concentration inequality for a sum of independent sub-exponential random variables.
\begin{theorem} \label{thm:bernstein}
For independent mean-zero sub-exponential random variables $\phi_1,\ldots,\phi_m$, for every $\eta>0$, we have
\begin{align*}
\Pr\left[\Big\lvert \sum_{i=1}^m  \phi_i \Big\rvert \ge \eta\right]\le 2\exp\left(-c \min\left\{\frac{\eta^2}{\sum_{i=1}^m \lv \phi_i\rv_{\psi_1}^2},\frac{\eta}{\max_i \lv \phi_i \rv_{\psi_1}}\right\}\right),
\end{align*}
where $c$ is
a positive absolute constant.
\end{theorem}

%Next is the Gaussian-Lipschitz contraction inequality applied to control the squared norm of a Gaussian random vector \citep[see, e.g.,][Example~2.28]{wainwright2019high}. 
% \begin{theorem} \label{thm:gaussian_concentration}
% Let $\phi_1,\ldots,\phi_m$ be drawn i.i.d. from $\cN(0,\sigma^2)$ then, for every $\eta>0$, we have
% \begin{align*}
% \Pr\left[ \sum_{i=1}^m  \phi_i^2  \ge \sigma^2 m (1+\eta)^2\right]\le \exp\left(-c m\eta^2  \right),
% \end{align*}
% where $c$ is
% a positive absolute constant.
% \end{theorem}

Let us continue by defining an $\epsilon$-net with respect to the Euclidean distance.
\begin{definition} \label{def:epsilon_net}Let $S \subseteq \mathbb{R}^p$. A subset $K$ is called an $\epsilon$-net of $S$ if every point in $S$ is within Euclidean distance $\epsilon$ of some point in $K$. 
\end{definition}
The following lemma bounds the size of a $1/4$-net of unit vectors in $\mathbb{R}^p$.
\begin{lemma}\label{l:covering_numbers_unit_vectors} 
Let $S$ be the set of all unit vectors in $\mathbb{R}^p$. Then there exists a $1/4$-net of $S$ of size $9^{p}$.
\end{lemma}
\begin{proof}
Follows immediately by invoking \citep[Corollary~4.2.13]{vershynin2018high} with $\epsilon=1/4$.
\end{proof}

\subsection{Proof of Lemma~\ref{l:eigenvalue_bound}}\label{s:eigenvalue_bound}
Let $\Sigma = \sum_{i=1}^p \lambda_i e_ie_i^{\top}$ be the spectral decomposition of the covariance matrix. Define the random vectors $$z_i := \frac{Xe_i}{\sqrt{\lambda_i}} \in \R^n.$$ These random vectors $z_i$ have entries that are independent, $\sigma_x^2$-sub-Gaussian random variables~\citep[see][Lemma~8]{bartlett2020benign}. Note that we can write the matrix
\begin{align*}
    X_{S}X_{S}^{\top}  = \sum_{i\in S} \lambda_i z_i z_i^{\top}.
\end{align*}
Further, its expected value is as follows:
\begin{align*}
    \E\left[X_{S}X_{S}^{\top}\right] & = \sum_{i\in S} \lambda_i \E\left[z_i z_i^{\top}\right]  =  \sum_{i\in S} \E\left[Xe_ie_i^{\top}X^{\top}\right]
     = I_n\sum_{i\in S} \lambda_i = I_{n} s(S).
    \end{align*}
    With this in place, we are now ready to prove our concentration results.
\eigenvaluebound*
\begin{proof} We shall prove this bound in the case where the set $S= [p]$. The bound for any other subset $S$ shall follow by exactly the same logic. First, note that by a standard $\epsilon$-net argument~\citep[see, e.g,][Lemma~25]{bartlett2020benign} to bound the operator norm we can use the following inequality:
\begin{align}\label{e:epsilon_net_argument}
    \left\lv XX^{\top} - I_n\sum_{i=1}^p \lambda_i \right\rv_{op} &\le 2\max_{v_j \in \cN_{\frac{1}{4}}} \left|v_j^{\top}\left(XX^{\top} - I_n\sum_{i=1}^p \lambda_i \right)v_j\right|,
\end{align}
where $\cN_{\frac{1}{4}}$ is a $1/4$-net of the unit sphere with respect to the Euclidean norm of size at most $9^{n}$. (We know that such a net exists by Lemma~\ref{l:covering_numbers_unit_vectors}.) Consider an arbitrary unit vector $v \in \S^{n-1}$. Then
\begin{align}
    v^{\top}\left(X X^{\top}-I_n\sum_{i=1}^p \lambda_i \right)v & = \sum_{i =1}^p \lambda_i \left((z_i^{\top}v)^{2} -1\right). \label{e:w_sphere_conc}
\end{align}
By Lemmas~\ref{l:sub_gaussian_squared} and \ref{l:sub_exponential_centering} we know that the random variables $\lambda_i((z_i^{\top}v)^{2} -1)$ are $c_1\lambda_i\sigma_x^2$-sub-exponential, for some positive constant $c_1$.% Therefore, they are all $\lambda_1 \sigma_x^2$-sub-exponential. 
Therefore we can use Bernstein's inequality (see Theorem~\ref{thm:bernstein}) to upper bound the sum in equation~\eqref{e:w_sphere_conc} to get that, with probability at least $1-2e^{-t}$,
\begin{align} \label{e:eigen_bound_bernstein_inequality_used}
    \left|\sum_{i =1 }^p \lambda_i \left((z_i^{\top}v)^{2} -1\right) \right| &\le c_2 \sigma_x^2\max\left\{\lambda_1t,\sqrt{t \sum_{j=1}^p \lambda_j^2}\right\}.
\end{align}
Next by a union bound over all the elements of the cover $\cN_{\frac{1}{4}}$ we find that, with probability at least $1-2e^{-t}$, for all $v \in \cN_{\frac{1}{4}}$,
\begin{align*}
     \left|\sum_{i =1 }^p \lambda_i \left((z_i^{\top}v)^{2} -1\right) \right| &\le c_2 \sigma_x^2\max\left\{\lambda_1\left(t+n\log(9)\right),\sqrt{\left(t+n\log(9)\right) \sum_{j=1}^p \lambda_j^2}\right\}.
\end{align*}
Hence, by using inequality~\eqref{e:epsilon_net_argument} we get that with probability at least $1-2e^{-t}$
\begin{align*}
     \left\lv XX^{\top} - I_n\sum_{i=1}^p \lambda_i \right\rv_{op} &\le c_3 \sigma_x^2\max\left\{\lambda_1\left(t+n\log(9)\right),\sqrt{\left(t+n\log(9)\right) \sum_{j=1}^p \lambda_j^2}\right\}\\
     & \le  c_4 \sigma_x^2 \left(\lambda_1 (t+n)+\sqrt{(t+n)\sum_{j=1}^p \lambda_j^2}\right).
\end{align*}
Recalling that $\sigma_x$ is assumed to be a positive constant, 
% and hence the first claim follows by simply using the operator norm bound to establish bounds on the extreme eigenvalues and because of the definition of the ranks as follows $r_0 = \frac{\sum_{j>0}\lambda_j}{\lambda_1}$ and $R_0 = \frac{(\sum_{j>0}\lambda_j)^2}{\sum_{j>0}\lambda_j^2}$.
%
% and noting that $\sum_{i=1}^p \lambda_i$ is the only
% eigenvalue of
% $I_n\sum_{i=1}^p \lambda_i$, 
this implies that the
greatest and least eigenvalues of
$XX^{\top}$ are within 
$c_5 \left(\lambda_1 (t+n)+\sqrt{(t+n)\sum_{j=1}^p \lambda_j^2}\right)$ of $\sum_{i=1}^p \lambda_i$,
which in turn implies
\begin{align*}
\left| \mu_j(XX^{\top}) - \sum_{i=1}^p \lambda_i 
 \right| & \leq c_5 \left(\lambda_1 (t+n)+\sqrt{(t+n)\sum_{j=1}^p \lambda_j^2}\right) \\
         & = c_5 \left( \sum_{i=1}^p \lambda_i \right)
             \left(\frac{t+n}{r_0}+\sqrt{\frac{t+n}{R_0}}\right),
\end{align*}
completing the proof.
\end{proof}

\subsection{Proof of Lemma~\ref{l:smallest_singular_value}} \label{s:smallest_multiplicative_bound_appendix}
We begin by proving an auxiliary lemma that relates the minimum singular value of a matrix to its approximation over an $\epsilon$-net under the assumption that its operator norm is bounded.
Recall that $X_S\in\mathbb{R}^{n\times|S|}$.
\begin{lemma}
\label{l:net_to_sphere} Let $\cN_{\epsilon}$ be an $\epsilon$-net of the unit sphere in $\R^{n}$ with respect to the Euclidean norm. For any $a,b\ge0$, if
\begin{align*}
    \inf_{z \in \mathbb{S}^{n-1}} \lv X^{\top}_S z \rv \le a-\epsilon b \quad \text{and} \quad \lv X^{\top}_S\rv_{op}\le b
\end{align*}
then $\inf_{z\in \cN_{\epsilon}}\lv X^{\top}_Sz\rv\le a$.
\end{lemma}
\begin{proof}
Let $\zeta$ be a function that maps any unit vector $z$ to its nearest neighbor (with respect to the Euclidean norm) in the net $\cN_{\epsilon}$. Therefore, if $\lv X^{\top}_S\rv_{op}\le b$ then
\begin{align*}
\inf_{z \in \mathbb{S}^{n-1}}\lv X^{\top}_Sz\rv 
    &= \inf_{z \in \mathbb{S}^{n-1}}\lv X^{\top}_S(z-\zeta(z)) +X^{\top}_S\zeta(z)\rv \\
     &\ge \inf_{z\in \mathbb{S}^{n-1}}\lv X^{\top}_S\zeta(z)\rv-\inf_{z \in \mathbb{S}^{n-1}}\lv X^{\top}_S(z-\zeta(z))\rv \\
    &= \inf_{z\in \cN_{\epsilon}}\lv X^{\top}_Sz\rv-\inf_{z \in \mathbb{S}^{n-1}}\lv X^{\top}_S(z-\zeta(z))\rv \\
    &\ge \inf_{z\in \cN_{\epsilon}}\lv X^{\top}_Sz\rv-\lv X^{\top}_S\rv_{op} \inf_{z \in \mathbb{S}^{n-1}}\lv z-\zeta(z)\rv\\
    &\ge \inf_{z\in \cN_{\epsilon}}\lv X^{\top}_Sz\rv-\epsilon b.
\end{align*}
Further if $\inf_{z \in \mathbb{S}^{n-1}}\lv X_S^{\top}z\rv \le a -\epsilon b$ then, due to the inequality above, $\inf_{z\in \cN_{\epsilon}}\lv X_S^{\top}z\rv\le a$ which completes the proof.
\end{proof}

With this lemma in place let us prove our result. 

\smallestsingularvalue*
\begin{proof} To reduce notational burden in the proof we shall present a proof in the case where the $S=[p]$, and therefore $X_S =X$. For any other subset $S$ the proof shall proceed in exactly the same manner. 

% First note that $r_{S}\max_{i\in S} \lambda_i = \sum_{i \in S}\lambda_i$. 
In the proof we shall prove bounds on the smallest singular value of $X$, $s_{\min}(X)$. This immediately leads to a bound on %$\mu_{\min(n,|S|)}(X_{S}X_{S}^{\top})$ since 
$\mu_{n}(XX^{\top}) = s^2_{\min}(X)$. 

Recall that $s_{\min}(X) = s_{\min}(X^{\top})$. So we will instead prove a bound on the smallest singular value of $X^{\top}$ to simplify our calculations. For some parameter $h\ge c_4\ge 1$ that will be set in the sequel, decompose the probability into
\begin{align*}
    &\Pr\left[s_{\min}(X^{\top}) \le t \sqrt{\sum_{j=1}^p\lambda_j}\right] \\& = \Pr\left[\left\{s_{\min}(X^{\top}) \le t \sqrt{\sum_{j=1}^p\lambda_j}\right\}\; \cap \; \left\{\lv X\rv_{op} \le h\sqrt{ \lambda_1 p}\right\}\right]\\&\qquad +\Pr\left[\left\{s_{\min}(X^{\top}) \le t \sqrt{\sum_{j=1}^p\lambda_j}     \right\}\; \cap \; \left\{\lv X\rv_{op} > h\sqrt{ \lambda_1 p}\right\}\right] \\
    & \le \Pr\left[\left\{s_{\min}(X^{\top}) \le t \sqrt{\sum_{j=1}^p\lambda_j}\right\}\; \cap \; \left\{\lv X\rv_{op} \le h\sqrt{ \lambda_1 p}\right\}\right]+ \Pr\left[\lv X\rv_{op} > h\sqrt{ \lambda_1 p}\right]. \numberthis \label{e:decomposition_into_two_pieces}
    %&\le \Pr\left[\left\{s_{\min}(X^{\top}) \le t \sqrt{\sum_{j=1}^p\lambda_j}\right\}\; \cap \; \left\{\lv X\rv_{op} \le r\lambda_1 \sqrt{p}\right\}\right]+ e^{-c_5 r^2 p} 
\end{align*}

Now we will control each of these probabilities separately. First, let us control the second probability
\begin{align*}
    \Pr\left[\lv X\rv_{op} > h \sqrt{ \lambda_1 p}\right]& = \Pr\left[\lv X\Sigma^{-1/2} \Sigma^{1/2}\rv_{op} > h \sqrt{ \lambda_1 p}\right] \\
    & \le \Pr\left[\lv X\Sigma^{-1/2} \rv_{op} > h \sqrt{p}\right]  \le e^{-c_4 h^2 p}, \label{e:upper_bound_multiplicative_singular_value} \numberthis
\end{align*}
%where in the last inequality follows
by invoking Proposition~2.4 by \citet{rudelson2009smallest}. 

To control the first probability in inequality~\eqref{e:decomposition_into_two_pieces} we need the following definition. Given a random vector $\xi \in \R^{p}$ define the L\'evy concentration function 
\begin{align*}
    \cL(\xi;t) := \sup_{w\in \R^p} \Pr\left[\lv \xi - w\rv\le t\right].
\end{align*}
Let $\phi\in \mathbb{S}^{n-1}$ be a fixed unit vector. By Assumption~\ref{assumption:small_ball_probability} we know that for any $a\le b \in \R$:
\begin{align}\label{e:projection_small_ball_property}
    \Pr\left[(\Sigma^{-1/2}X^{\top}\phi)_i \in [a,b]\right] \le c|b-a|.
\end{align}
Using this fact we find that for any $i \in [p]$:
\begin{align*}
    \cL((\Sigma^{-1/2}X^{\top}\phi)_i\; ; 2t) &= \sup_{w\in \R} \Pr\left[| (\Sigma^{-1/2}X^{\top}\phi)_i - w|\le 2t\right]\\&=\sup_{w\in \R} \Pr\left[(\Sigma^{-1/2}X^{\top}\phi)_i
    \in [w-2t,w+2t]\right]\le 4c_5t.
\end{align*}
Next by invoking Theorem~1.5 in \cite{rudelson2015small} we infer that
\begin{align*}
    \cL\left(X^{\top}\phi; 2t\sqrt{\sum_{i=1}^p \lambda_i}\right) \le (ct)^{c'r_0}.
\end{align*}
This implies that
\begin{align*}
    \Pr\left[\lv X^{\top}\phi\rv \le 2t\sqrt{\sum_{i=1}^p \lambda_i}\right] &\le \sup_{w\in \R^{p}}\Pr\left[\lv X^{\top}\phi-w\rv \le 2t\sqrt{\sum_{i=1}^p \lambda_i}\right] \\&= \cL\left(X^{\top}\phi;2t\sqrt{\sum_{i=1}^p \lambda_i}\right)\le (ct)^{c'r_0}. \numberthis \label{e:small_ball_probability_single_vector}
\end{align*}
This establishes a \emph{small-ball} probability (anti-concentration) for a fixed unit vector $\phi$. We will now proceed by using an $\epsilon$-net argument. For some $\epsilon \in \Big(0, \frac{2t}{h}\sqrt{\frac{\sum_{i=1}^p\lambda_i}{\lambda_1 p}}\Big)$ let $\cN_{\epsilon}$ be an $\epsilon$-net of the unit vectors in $\R^{n}$ with respect to the Euclidean norm of size at most $\left(\frac{2}{\epsilon}+1\right)^n$ (such a net exists, see, e.g., Corollary~4.2.13 in \cite{vershynin2018high}). By a union bound over the elements of the net
\begin{align*}
    \Pr\left[\min_{\phi \in \cN_{\epsilon}}\lv X^{\top}\phi\rv \le 2t\sqrt{\sum_{i=1}^p \lambda_i}\right] \le (ct)^{c'r_0} \cdot \left(\frac{2}{\epsilon}+1\right)^n. \numberthis \label{e:net_guarantees} 
\end{align*}
Next by Lemma~\ref{l:net_to_sphere} we know that 
\begin{align*}
    &\Pr\left[\left\{s_{\min}(X^{\top}) \le 2t \sqrt{\sum_{i=1}^p \lambda_i} -\epsilon h\sqrt{\lambda_1 p}\right\}\; \cap \; \left\{\lv X\rv_{op} \le h\sqrt{\lambda_1 p}\right\}\right] \\
    &\quad=\Pr\left[\left\{\inf_{z \in \mathbb{S}^{n-1}}\lv X^{\top}z\rv \le 2t \sqrt{\sum_{i=1}^p \lambda_i}-\epsilon h\sqrt{\lambda_1 p}\right\}\; \cap \; \left\{\lv X\rv_{op} \le h\sqrt{\lambda_1 p}\right\}\right] \\
    &\quad\le \Pr\left[\min_{z\in \cN_{\epsilon}}\lv X^{\top}z\rv \le 2t\sqrt{\sum_{i=1}^p \lambda_i}\right] \\&\quad\le (ct)^{c'r_0} \cdot \left(\frac{2}{\epsilon}+1\right)^n.
\end{align*}
Setting $\epsilon = \frac{t}{h}\sqrt{\frac{\sum_{i=1}^p \lambda_i}{\lambda_1 p}} = \frac{t}{h}\sqrt{\frac{r_0}{p}}  $ we get that
\begin{align*}
    \Pr\left[\left\{s_{\min}(X^{\top}) \le t \sqrt{\sum_{i=1}^p \lambda_i}\right\}\; \cap \; \left\{\lv X\rv_{op} \le h\sqrt{\lambda_1 p}\right\}\right]&\le (ct)^{c'r_0}\cdot \left(\frac{2h}{t}\sqrt{\frac{p}{r_0}}+1\right)^n.
\end{align*}
This combined with inequalities~\eqref{e:decomposition_into_two_pieces} and \eqref{e:upper_bound_multiplicative_singular_value} above yields
\begin{align*}
\Pr\left[s_{\min}(X^{\top}) \le t \sqrt{\sum_{i=1}^p \lambda_i}\right]  &\le (ct)^{c'r_0}\cdot \left(\frac{2 h}{t}\sqrt{\frac{p}{r_0}}+1\right)^n
% +e^{-c_5 r^2 p}.
+e^{-c_4 h^2 p}.
\end{align*}
Finally set $h=\frac{1}{t}\sqrt{\frac{r_0}{p}}$ to obtain the bound
\begin{align*}
  \Pr\left[s_{\min}(X^{\top}) \le t \sqrt{\sum_{i=1}^p \lambda_i}\right]&\le (ct)^{c'r_0}\cdot \left(\frac{c''}{t^2}\right)^n+e^{-c_5 r_0/t^2} %\overset{(i)}{\le} (c_2 t)^{c_3 \min\{p,r_0\}}
  \overset{(i)}{\le} (c_2 t)^{c_3 \cdot r_0}
\end{align*}
where $(i)$ follows since $r_0>c_0 n$ for a large enough constant $c_0$ and because $t<c_1$ for a small enough constant $c_1$. %and $(ii)$ follows since $r_0 = \sum_{i=1}^p (\lambda_i/\lambda_1) \le p$.
\end{proof}
\printbibliography

\end{document}